\documentclass[11pt]{article}

\usepackage[preprint]{acl}

\usepackage{times}         
\usepackage{latexsym}
\usepackage[T1]{fontenc}
\usepackage[utf8]{inputenc}
\usepackage{microtype}
\usepackage{amsmath}
\usepackage{amsthm}
\usepackage{amssymb}
\usepackage{booktabs}

\usepackage{latexsym}
\usepackage[T1]{fontenc}
\usepackage[utf8]{inputenc}
\usepackage{microtype}
\usepackage{inconsolata}

\usepackage{amsmath}
\usepackage{amsfonts}
\usepackage{amssymb}
\usepackage{amsthm}
\usepackage{mathtools}
\usepackage{nicefrac}

\newtheorem{theorem}{Theorem}[section]
\newtheorem{corollary}{Corollary}
\newtheorem{lemma}{Lemma}
\theoremstyle{definition}
\newtheorem{definition}{Definition}[section]
\theoremstyle{remark}
\newtheorem*{remark}{Remark}

\usepackage{booktabs}
\usepackage{multirow}
\usepackage{graphicx}
\usepackage{caption}
\usepackage{subcaption}
\usepackage{array}
\usepackage{float}
\usepackage{url}

\theoremstyle{definition}

\usepackage{tikz}
\usetikzlibrary{shapes,arrows,positioning,calc,decorations.pathreplacing}

\tikzstyle{block} = [rectangle, draw, fill=blue!10, 
    text width=5em, text centered, rounded corners, minimum height=2.5em, font=\scriptsize]
\tikzstyle{reward} = [rectangle, draw, fill=green!20, 
    text width=5.5em, text centered, rounded corners, minimum height=2em, font=\scriptsize]
\tikzstyle{policy} = [rectangle, draw, fill=orange!20, 
    text width=4em, text centered, rounded corners, minimum height=1.8em, font=\scriptsize]
\tikzstyle{arrow} = [thick,->,>=stealth]
\tikzstyle{bidarrow} = [thick,<->,>=stealth,red]
\tikzstyle{gradient_arrow} = [thick,<-,>=stealth]
\tikzstyle{query} = [rectangle, draw, fill=green!15, 
    minimum width=0.8cm, minimum height=1cm, rounded corners, font=\small]
\tikzstyle{sample} = [rectangle, draw, fill=blue!15, 
    minimum width=3cm, minimum height=1cm, rounded corners, font=\small]
\tikzstyle{decision} = [rectangle, draw, fill=orange!15, 
    minimum width=0.8cm, minimum height=1cm, rounded corners, font=\small]

\title{The Two-Stage Decision-Sampling Hypothesis: Understanding the Emergence of Self-Reflection in RL-Trained LLMs}

\author{
  \textbf{Zibo Zhao\textsuperscript{1}},
  \textbf{Haipeng Zhang\textsuperscript{2,*}},
  \textbf{Yuanting Zha\textsuperscript{2}},
  \textbf{Xingcheng Xu\textsuperscript{3,*}}
\\
\\
  \textsuperscript{1}Arizona State University,
  \textsuperscript{2}Shanghaitech University,
  \textsuperscript{3}Shanghai Artificial Intelligence Laboratory
\\
  \texttt{zzhao203@asu.edu}, \texttt{zhanghp@shanghaitech.edu.cn} \\
  \texttt{zhayt2022@shanghaitech.edu.cn}, \texttt{xingcheng.xu18@gmail.com}
}

\tikzstyle{block} = [rectangle, draw, fill=blue!10, 
    text width=5em, text centered, rounded corners, minimum height=2.5em, font=\scriptsize]
\tikzstyle{reward} = [rectangle, draw, fill=green!20, 
    text width=5.5em, text centered, rounded corners, minimum height=2em, font=\scriptsize]
\tikzstyle{policy} = [rectangle, draw, fill=orange!20, 
    text width=4em, text centered, rounded corners, minimum height=1.8em, font=\scriptsize]
\tikzstyle{arrow} = [thick,->,>=stealth]
\tikzstyle{bidarrow} = [thick,<->,>=stealth,red]
\tikzstyle{gradient_arrow} = [thick,<-,>=stealth]
\tikzstyle{query} = [rectangle, draw, fill=green!15, 
    minimum width=0.8cm, minimum height=1cm, rounded corners, font=\small]
\tikzstyle{sample} = [rectangle, draw, fill=blue!15, 
    minimum width=3cm, minimum height=1cm, rounded corners, font=\small]
\tikzstyle{decision} = [rectangle, draw, fill=orange!15, 
    minimum width=0.8cm, minimum height=1cm, rounded corners, font=\small]

\begin{document}

\maketitle
\begingroup
  \def\thefootnote{*}\footnotetext{Corresponding authors.}
\endgroup
\begin{abstract}

Self-reflection capabilities emerge in Large Language Models after RL post-training, with multi-turn RL achieving substantial gains over SFT counterparts. Yet the mechanism of how a unified optimization objective gives rise to functionally distinct capabilities of generating solutions and evaluating when to revise them remains opaque. To address this question, we introduce the Gradient Attribution Property to characterize how reward gradients distribute across policy components, formalized through the Two-Stage Decision-Sampling (DS) Hypothesis, which decomposes the policy into sampling ($\pi_{sample}$) for generation and decision ($\pi_{d}$) for verification. We prove that surrogate rewards exhibit Balanced Gradient Attribution, while SFT and KL penalties exhibit Unbalanced Gradient Attribution, with length-weighting creating asymmetric regularization that constrains $\pi_{sample}$ while leaving $\pi_{d}$ under-optimized, providing a theoretical explanation of why RL succeeds where SFT fails. We also empirically validate our theoretical predictions on arithmetic reasoning, demonstrating that RL's superior generalization stems primarily from improved decision-making ($\pi_{d}$) rather than sampling capabilities, providing a first-principles mechanistic explanation for self-correction in thinking models.

\end{abstract}

\section{Introduction}

The ability to verify reasoning, detect errors, and revise incorrect answers (self-reflection capabilities) emerge spontaneously in large language models after RL post-training \citep{deepseek2025r1, openai2024openaio1card, bercovich2025llamanemotronefficientreasoningmodels, zhao2024marcoo1openreasoningmodels}. This emergent behavior correlates strongly with substantial performance improvements, particularly in mathematical reasoning. Yet despite widely documented gains, the mechanism by which RL produces these qualitatively different capabilities remains theoretically opaque. It is unclear how a unified reinforcement learning objective gives rise to functionally distinct abilities of generating candidate solutions versus evaluating when to accept or revise them and why RL succeeds where supervised fine-tuning (SFT) consistently fails.

While existing literature extensively records that the emergence of verification and revision behaviors occurs, the literature to the date lacks an anatomical explanation of how the training objective fundamentally alters the model's policy. Specifically, it remains unclear which optimization processes drive the functional separation required for self-correction. To bridge this gap, we must examine the underlying gradient dynamics that govern this behavioral bifurcation.

A critical preliminary distinction sharpens the puzzle. Both RL-trained and SFT-trained models can produce verification-like token sequences (phrases such as ``Let me recheck'' or ``I made a mistake''), yet only RL-trained models exhibit \textit{functional} verification: actually changing answers when wrong and retaining them when correct. \cite{huang2024llm} and \cite{kamoi2024critical} document that intrinsic self-correction in SFT models is largely illusory; \cite{kang2025trymattersrevisitingrole} test this directly by varying the proportion of corrective transitions in SFT training data and find that $p(F\!\to\!T)$ remains flat at 3--5\% regardless of data composition (Table~\ref{tab:kang2025}). The models produce reflective text fluently, but this text does not lead to actual error correction. The question, therefore, is not whether a model generates reflective tokens, but whether reflection leads to discriminative action, a distinction our framework formalizes through the decomposition below.

In this paper, we view learning to self-reflect as learning to multitask with a shared policy function. Under this framework, we introduce the Gradient Attribution Property to characterize how reward gradients distribute across policy components. We formalize this through the Two-Stage Decision-Sampling (DS) Hypothesis, which decomposes the model's unified policy into a sampling policy $\pi_{sample}$ for content generation and a decision policy $\pi_d$ for verification and stopping. When a model develops self-reflection, it primarily learns to improve $\pi_d$( judgment about when to trust versus revise outputs) rather than (or in addition to) improving $\pi_{sample}$. Transforming the question about "emergent self-reflection" into a question about gradient flow: how does the RL objective differentially update these policy components?

We prove that different reward structures induce fundamentally different learning dynamics. Surrogate rewards exhibit balanced gradient attribution: variations in the reward signal map cleanly to whichever policy component was responsible, creating symmetric learning pressure on both $\pi_{sample}$ and $\pi_d$. In contrast, KL-divergence penalties exhibit imbalanced gradient attribution: length-weighting in token-level calculations creates asymmetric regularization that heavily constrains $\pi_{sample}$ while leaving $\pi_d$ relatively unconstrained. This explains why RL succeeds, its objective mathematically favors learning better $\pi_d$, while SFT, which resembling a KL-divergence objective without countervailing reward, systematically fails to develop genuine self-reflection.

We make several contributions. First, we formalize the DS-Hypothesis and mathematically characterize when gradient signals can be cleanly attributed to specific policy components. Second, we provide a first-principles mechanistic account of why RL algorithms like GRPO induce self-reflection while SFT does not. Third, we empirically validate the framework on arithmetic reasoning tasks, demonstrating that: (i) the framework accurately predicts model performance; (ii) RL improves $\pi_d$ more than $\pi_{sample}$; and (iii) out-of-distribution generalization is primarily limited by $\pi_{sample}$. Finally, we explain documented phenomena including the "echoing effect"\footnote{\cite{kang2025trymattersrevisitingrole} found that SFT-trained thinking models never discriminate between correct and incorrect answers; post-reflection answers largely echo pre-reflection ones.} and why reflection-rich SFT data improves first-answer accuracy without enabling genuine self-correction.

The remainder proceeds as follows. Section 2 reviews policy gradient methods, mathematical reasoning generalization, and the evolution of self-correction. Section 3 develops the theoretical framework. Section 4 presents empirical validation. Section 5 applies insights to explain SFT's limitations. Section 6 concludes.

\section{Related Work}\label{sec:related_work}

Our work lies at the intersection of three research streams: the control-theoretic foundations of policy gradients, mathematical generalization in Transformers, and the evolution of self-correction from prompting to learned RL behaviors. To develop the Gradient Attribution Property framework: we borrow analytical tools for decomposing reward gradients from policy gradient theory; we adopt arithmetic tasks as a controlled testbed from the mathematical reasoning literature. Then we take the empirical puzzle from the self-correction literature to motivate the setting of the DS-Hypothesis and showed that our theoretical framework can theoretically explain these puzzles. Our contribution is showing that gradient attribution properties of training objectives determine whether models learn genuine decision-making or merely imitate reflective patterns. We review each literature stream below.


\subsection{Foundations and Practices of Policy Optimization Methods}

The transition from supervised imitation to reinforcement learning in reasoning models rests fundamentally on policy gradient theory. \cite{sutton1999policy} introduced \textit{options}\footnote{Sutton's seminal work temporally extended actions within semi-Markov decision processes.} providing theoretical justification for treating multi-token Chain of Thought as optimizable sequences rather than mere predictions; \cite{agarwal2020theorypolicygradientmethods} reviews the theory and convergence properties.\footnote{We do not address convergence here; this motivates our single-dimension demonstration.} Policy gradient methods have been widely applied in LLM post-training \citep{ouyang2022traininglanguagemodelsfollow, stiennon2020learning}, with prominent algorithms including TRPO \citep{schulman2017trustregionpolicyoptimization}, PPO \citep{schulman2017proximalpolicyoptimizationalgorithms}, and DPO \citep{rafailov2024directpreferenceoptimizationlanguage}.\footnote{TRPO constrains updates via KL-divergence; PPO simplifies this through clipping and has become standard for RLHF.} Most relevant to our setting, \cite{shao2024deepseekmath} introduced GRPO, which eliminates value networks by normalizing rewards within group samples, enabling thinking models with emergent self-reflection \citep{deepseek2025r1}.\footnote{Our theoretical analysis abstracts away clipping and trust-region constraints to preserve tractability. Although GRPO is closest to our setting, our framework generalizes across policy-based algorithms rather than formalizing any specific one.}

\subsection{Mathematical Reasoning and Generalization}

Arithmetic reasoning serves as a rigorous testbed for Transformer generalization. \cite{nogueira2021investigating} and \cite{anil2022exploring} established that standard Transformers, despite scale, exhibit catastrophic length extrapolation failures, suggesting they learn surface heuristics rather than robust algorithms. \cite{lee2023teaching} challenged the necessity of scale, showing small models can master arithmetic via data formatting (scratchpads, reverse-order generation), with sharp phase transitions indicative of "grokking." \cite{xu2025principled} formalized these mechanics through a unified framework linking generalization to architecture-task symmetry alignment. Beyond simple arithmetic, benchmarks like GSM8K, MATH, and MetaMathQA are widely used to measure LLM reasoning ability \citep{cobbe2021trainingverifierssolvemath, hendrycks2021measuringmathematicalproblemsolving, yu2024metamathbootstrapmathematicalquestions}. We adopt mathematical reasoning for its clarity in evaluating learning effects.\footnote{Ideally, our framework would apply more generally; developing valid measurements in other domains remains challenging.}

\subsection{RL-Reflection and Self-Correction}

Self-correction capabilities have evolved from prompt engineering to learned behaviors. While \cite{wei2022chain} and \cite{feng2023revealingmysterychainthought} established the computational necessity of intermediate reasoning steps, critical surveys by \cite{kamoi2024critical} and \cite{huang2024llm} revealed that "intrinsic self-correction" in SFT models is often illusory or Oracle-dependent. \cite{kumar2024score} (SCoRe) identified SFT's distribution mismatch (training on others' mistakes versus correcting one's own) and proposed multi-turn RL on self-generated traces. Subsequent frameworks including Self-Rewarding Correction \citep{xiong2025selfrewardingcorrectionmathematicalreasoning} and PAG \citep{jiang2025pagmultiturnreinforcedllm} unify solver and critic into a single policy. Most relevant, \cite{zhao2025boostingllmreasoningspontaneous} and \cite{ma2025s2rteachingllmsselfverify} explicitly decompose policies into answer generation and verification, demonstrating performance gains from this breakdown. We focus instead on the \textit{implicit} policy decomposition driving emergent self-reflection in models like DeepSeek-R1.

\section{Formalization of The Decision-Sampling Hypothesis}\label{sec:formalization}
We develop a formal framework explaining why RL produces self-reflection while SFT fails. The key concept is the \textbf{gradient attribution property}, characterizing how reward gradients distribute across policy components. When a model performs multiple functions through a single policy, balanced gradient attribution enables learning all functions effectively, while imbalanced attribution causes some functions to be learned ineffectively.
We model an LLM solving query $Q$ as a sequential decision process where at each iteration $k \geq 1$, the model: (i) samples a candidate answer $A_k$ with reasoning $T_k$, then (ii) decides whether to STOP or RESAMPLE. This decomposition follows the \textit{options framework} \citep{suttonetal1999}: $\pi_{\text{sample}}$ corresponds to the intra-option policy (executing a temporally extended action, generating a reasoning trace and candidate answer), while $\pi_d$ corresponds to the termination condition (deciding whether to accept the current answer or initiate another attempt). The decomposition is analytical, not architectural, the model remains a single next-token predictor throughout, but it enables analysis of how gradient signals distribute across functionally distinct phases of generation. We prove that surrogate rewards\footnote{We analyze simple surrogate rewards for tractability; practice uses clipped variants for stability.} exhibit balanced attribution while KL-penalties exhibit imbalanced attribution—explaining why SFT-trained models lack genuine self-reflection.

Section 3.1 introduces our formal setting; Section 3.2 defines gradient attribution; Section 3.3 analyzes both reward components. Section 5 applies these results to SFT behavior.

\subsection{Settings and Preliminary}

\textbf{State Space} The state at step $k$ is $s_k = (Q, A_k, T_k)$, encoding the query and the model's current candidate solution.

\textbf{Policy Decomposition} The overall policy $\pi_\theta$ decomposes into two components:
\begin{itemize}
    \item Sampling policy $\pi_{\text{sample}}(\cdot | s_{k-1}; \theta)$: distribution over $(A_k, T_k)$ given context
    \item Decision policy $\pi_d(\cdot | s_k; \theta)$: distribution over ${\text{STOP}, \text{RESAMPLE}}$ given current state
\end{itemize}

\noindent \textbf{Trajectory Probability Factorization:} A trajectory $\tau$ of length $T$ consists of a sequence of samples and decisions:
\begin{equation*}
\resizebox{0.95\hsize}{!}{$
\tau = (A_1, T_1, \text{RESAMPLE}, \dots, A_T, T_T, \text{STOP})
$}
\end{equation*}

\begin{lemma}
The probability of trajectory $\tau$ under policy $\pi_\theta$ factorizes as:
\begin{multline*}
P(\tau|Q; \theta) = \left[ \prod_{k=1}^{T} \pi_{\text{sample}}(A_k, T_k | s_{k-1}; \theta) \right] \\
\cdot \left[ \prod_{k=1}^{T-1} \pi_d(\text{R} | s_k; \theta) \right] \cdot \pi_d(\text{S} | s_T; \theta)
\end{multline*}
\end{lemma}

\begin{corollary}
The log-probability gradient separates into sampling and decision components:
\begin{equation*}
\begin{split}
\nabla_\theta \log P(\tau|Q; \theta) &= \sum_{k=1}^{T} \nabla_\theta \log \pi_{\text{sample}}(A_k, T_k | s_{k-1}) \\
&\quad + \sum_{k=0}^{T} \nabla_\theta \log \pi_d(a_k | s_k)
\end{split}
\end{equation*}
where $a_k \in \{\text{RESAMPLE}, \text{STOP}\}$ denotes the decision at step $k$.
\end{corollary}

\noindent \textbf{Reward Function:} For query $Q$ with ground truth answer $A^*_Q$, the reward of trajectory $\tau$ ending at step $T$ is $R(\tau) = \mathbb{I}(A_T = A^*_Q)$. Thus, for a given policy $\pi_\theta$ and reward function $R$, define:
\begin{equation*}
Q^\pi_R(s, a) = \mathbb{E}_{\pi_\theta} \left[ \sum_{k=t}^{\infty} \gamma^{k-t} R_k \,\Big|\, s_t = s, a_t = a \right]
\end{equation*}
where $\gamma \in (0, 1]$ is the discount factor.

\begin{lemma}
The gradient of the expected return decomposes as:
\begin{equation*}
\resizebox{1.0\hsize}{!}{$
\begin{split}
\nabla_\theta J(\theta) = \mathbb{E}_{\tau \sim \pi_\theta} \Bigg[ \sum_{t=0}^{T} \bigg( & \nabla_\theta \log \pi_{\text{sample}}(a'_t | s_{t-1}) Q^\pi_{\text{sample}} \\
& + \nabla_\theta \log \pi_d(a''_t | s_t) Q^\pi_d \bigg) \Bigg]
\end{split}
$}
\end{equation*}
where $a'_t = (A_t, T_t)$ denotes sampling actions and $a''_t \in \{\text{RESAMPLE}, \text{STOP}\}$ denotes decision actions.
\end{lemma}

\subsection{Gradient Attribution Property}

We formalize gradient attribution through the information structure of Q-values.

\begin{definition}[Gradient Attribution Property]
Consider a reward function $R$ and the induced Q-values $Q^\pi_{\text{sample}}(s, a')$ and $Q^\pi_d(s, a'')$ under policy $\pi_\theta$.

\noindent A reward $R$ exhibits \textbf{balanced gradient attribution}\footnote{\textit{Interpretation: Balanced attribution means both Q-values can be expressed in terms of the same information about future rewards (the sufficient statistic $\Phi$), just evaluated at different states along the trajectory. This allows the unified network to learn a consistent representation of "future value" that both $\pi_{\text{sample}}$ and $\pi_d$ can use. }
} if the Q-values admit a decomposition:
\[ Q^\pi_{\text{sample}}(s_{k-1}, a'_k) = f_{\text{sample}}(s_{k-1}, a'_k, \Phi(s_k)) \]
\[ Q^\pi_d(s_k, a''_k) = f_d(s_k, a''_k, \Phi(s_{k+1})) \]
where $\Phi: \mathcal{S} \to \mathbb{R}$ is a scale-invariant sufficient statistic and the weighting functions $f_{\text{sample}}, f_d$ are of comparable magnitude: $f_{\text{sample}} = \Theta(f_d)$

\noindent A reward exhibits \textbf{unbalanced gradient attribution}\footnote{\textit{Unbalanced attribution means the Q-values require different information about the future ($\Phi_{\text{sample}}$ vs $\Phi_d$). The unified network cannot learn a single coherent representation of future value—attempting to do so leads to "mismeshing" the gradient signals, where updates intended for one conceptual function interfere with learning the other.}
} if the Q-values decompose as:
\[Q_{\text{sample}}^{\pi}(s_{k-1}, a'_k) = r_k^{\text{sample}} + \gamma \cdot V^{\pi}(s_k) \]
\[Q_d^{\pi}(s_k, a''_k) = r_k^{\text{decision}} + \gamma \cdot V^{\pi}(s_{k+1}) \]
where the immediate reward components satisfy $|r_k^{\text{sample}}| = \omega(|r_k^{\text{decision}}|)$ systematically (i.e., scale-separated by more than a constant factor across typical trajectories).

\end{definition}

\begin{remark}[balanced attribution]
In the case of balanced attribution, we can identify $\Phi(s) = V^\pi(s)$ as the state value function. This is the standard Bellman decomposition:
\[ Q^\pi(s, a) = \mathbb{E}[R(s,a) + \gamma V^\pi(s') | s, a] \]
The key is that the \textit{same} $V^\pi$ appears in the decomposition for both policy components.
\end{remark}

\begin{remark}[When unbalanced attribution arise?]
Consider a reward structure where:
\begin{itemize}
    \item The immediate reward for sampling depends on sequence length: $r_{\text{sample}} \sim O(L_k)$
    \item The immediate reward for decisions is scalar: $r_d \sim O(1)$
    \item Future value recursively depends on these asymmetric immediate rewards
\end{itemize}
Then $\Phi_{\text{sample}}$ must encode "accumulated future sampling costs" while $\Phi_d$ encodes "accumulated future decision costs", which are fundamentally different scales. This is precisely what happens with the KL penalty, as we show in the section 3.3.
\end{remark}

\begin{remark}{Operational Identification of Decision Actions}
    Although the decomposition is analytical, decision actions can be operationalized through identifiable proxy tokens that signal verification or revision intent. In our experiments, we detect decision boundaries via lexical markers including "I'll go back and check", "Let me recompute", "I made a mistake" and similar verification phrases. The token sequence following such markers until the next candidate answer constitutes a decision action, while extended generation of reasoning and answers constitutes sampling actions. This operationalization enables the calibration reported in Section 4.2: we estimate $p_d|C$ and $p_d|W$ by observing stopping and resampling behavior conditional on answer correctness, where "resampling" is identified by the presence of revision markers followed by a new solution attempt
\end{remark}

\subsection{Gradient Attribution of Surrogate Reward and KL Penalty}
In this subsection, we analyze the gradient attribution properties of two fundamental reward functions used in RLM training.

\subsubsection{Simple Surrogate Reward Has Balanced Gradient Attribution}

The simple surrogate reward with advantage $A_i$ for trajectory $\tau_i$ is:
\[ L_{\text{reward}}(\theta) = \mathbb{E}_{\tau_i \sim \pi_{\text{old}}}\left[\frac{\pi_\theta(\tau_i|Q_i)}{\pi_{\text{old}}(\tau_i|Q_i)}A_i\right] \]
where $A_i$ is the group-relative advantage measuring whether trajectory $\tau_i$ is better or worse than the average trajectory for query $Q_i$.

\begin{theorem}
For the surrogate reward objective with trajectory-level advantage $A_i$, the Q-values satisfy:
\[Q_{\text{sample}}^{\pi, \text{reward}}(s_{k-1}, a'_k) = \gamma^{\sum_{j=k}^{T} \text{len}(A_j, T_j)} \cdot A_i\]
\[Q_d^{\pi, \text{reward}}(s_k, a''_k) = \gamma^{\sum_{j=k}^{T} \text{len}(A_j, T_j)} \cdot A_i\]
The sufficient statistic $\Phi(s_k) = \gamma^{\sum_{j=k}^{T} \text{len}(A_j, T_j)} A_i$ is \textbf{scale-invariant}: it enters both Q-values as an identical multiplicative factor. Consequently, the advantage $A_i$ weights gradient contributions to $\pi_{\text{sample}}$ and $\pi_d$ symmetrically.
\end{theorem}

\textbf{Proof Sketch.} The advantage $A_i$ is a trajectory-level scalar measuring overall quality. From the policy gradient theorem, the gradient is:
\begin{equation*}
\resizebox{1.0\hsize}{!}{$
\nabla_\theta \mathcal{L}_{\text{reward}} = \mathbb{E}_{\tau_i}\left[A_i \sum_{k=1}^T \nabla_\theta \log \pi_{\text{sample}}(\cdot) + A_i \sum_{k=0}^T \nabla_\theta \log \pi_d(\cdot)\right]
$}
\end{equation*}
The advantage multiplies both gradient components equally. Since both $\pi_{\text{sample}}$ and $\pi_d$ contributed to producing the trajectory that received advantage $A_i$, and the advantage reflects the \textit{coupled} outcome (correct answer + appropriate stopping), both Q-values equal the length-discounted advantage. The key observation is that the advantage $A_i$ is a trajectory-level scalar that does not decompose into action-specific components. From the policy gradient theorem, the advantage multiplies both gradient sums identically. There is no action-specific immediate reward that could create magnitude asymmetry, both policy components receive gradient signal proportional to their log-probability scores weighted by the \textit{same} scalar $A_i$. This is the formal sense in which attribution is "balanced." See Appendix~\ref{app:proof_corollary} for complete proof. 

\textbf{Interpretation.} The sufficient statistic $\Phi = \gamma^{\sum \text{len}(\cdot)} A_i$ encodes: "this trajectory will yield advantage $A_i$ at the end, discounted by the temporal distance." Both $\pi_{\text{sample}}$ and $\pi_d$ use this \textit{same} information when evaluating their actions. This enables coherent learning in the unified network.

\subsubsection{KL Penalty: Unbalanced Gradient Attribution}

In RLHF and GRPO, the KL divergence penalty regularizes the learned policy $\pi_\theta$ against a reference policy $\pi_{ref}$ (typically the SFT model). Following standard practice, we consider the token-level KL penalty added to the reward:\footnote{This corresponds to the ``KL in reward'' placement used in GRPO and PPO-based RLHF implementations.}
\begin{equation*}
\resizebox{0.8\hsize}{!}{$
    J(\theta) = \mathbb{E}_{\tau \sim \pi_\theta}\left[R(\tau) - w \sum_{t=1}^{|\tau|} \log \frac{\pi_\theta(a_t|s_t)}{\pi_{ref}(a_t|s_t)}\right]
$}
\end{equation*}
where $w \geq 0$ controls the regularization strength.

Under our policy decomposition, the per-trajectory KL penalty decomposes as:
\begin{equation*}
\resizebox{0.8\hsize}{!}{$
\begin{aligned}
D_{KL}^{\text{token}}(\tau) &= \underbrace{\sum_{k=1}^{T} \sum_{j=1}^{L_k} \log \frac{\pi_{\text{sample},\theta}(\text{token}_{k,j} \mid \cdot)}{\pi_{\text{sample},\text{ref}}(\text{token}_{k,j} \mid \cdot)}}_{\text{Sampling KL: } \sum_k L_k \text{ terms}} \\
&\quad + \underbrace{\sum_{k=1}^{T} \log \frac{\pi_{d,\theta}(a_k \mid s_k)}{\pi_{d,\text{ref}}(a_k \mid s_k)}}_{\text{Decision KL: } T \text{ terms}}
\end{aligned}
$}
\end{equation*}

The structural asymmetry is immediate: sampling actions contribute $\sum_{k=1}^{T} L_k$ terms while decision actions contribute only $T$ terms. Since typical reasoning traces have $L_k \gg 1$ (hundreds of tokens per attempt), the sampling component dominates.

Define the immediate KL penalties for each action type:
\begin{equation*}
\resizebox{0.8\hsize}{!}{$
\begin{aligned}
d_k^{\text{sample}} &= \sum_{j=1}^{L_k} \log \frac{\pi_{\text{sample},\theta}(\text{token}_{k,j} \mid \cdot)}{\pi_{\text{sample},\text{ref}}(\text{token}_{k,j} \mid \cdot)} \sim O(L_k) \\
d_k^{\text{decision}} &= \log \frac{\pi_{d,\theta}(a_k \mid s_k)}{\pi_{d,\text{ref}}(a_k \mid s_k)} \sim O(1)
\end{aligned}
$}
\end{equation*}

\begin{theorem}
For the token-level KL penalty, the Q-values satisfy the Bellman recursions:
\begin{equation*}
\resizebox{0.9\hsize}{!}{$
\begin{aligned}
Q_{\text{sample}}^{\pi,\text{KL}}(s_{k-1}, a'_k) &= d_k^{\text{sample}} + \gamma \cdot \mathbb{E}_{\pi_\theta} [Q_d^{\text{KL}}(s_k, a''_k)] \\
Q_d^{\pi,\text{KL}}(s_k, a''_k) &= d_k^{\text{decision}} + \gamma \cdot \mathbb{E}_{\pi_\theta} [Q_{\text{sample}}^{\text{KL}}(s_k, a'_{k+1})]
\end{aligned}
$}
\end{equation*}
where the immediate penalties satisfy:
\begin{equation*}
\resizebox{0.9\hsize}{!}{$
\begin{aligned}
d_k^{\text{sample}} &= \sum_{j=1}^{L_k} \log \frac{\pi_{\text{sample},\theta}(\text{token}_{k,j} \mid \cdot)}{\pi_{\text{sample},\text{ref}}(\text{token}_{k,j} \mid \cdot)} \sim O(L_k) \\
d_k^{\text{decision}} &= \log \frac{\pi_{d,\theta}(a_k \mid s_k)}{\pi_{d,\text{ref}}(a_k \mid s_k)} \sim O(1)
\end{aligned}
$}
\end{equation*}
The scale separation $|d_k^{\text{sample}}| / |d_k^{\text{decision}}| \approx L_k$ creates systematic gradient magnitude asymmetry.
\end{theorem}
\textit{Proof Sketch.} The key observation is that $d_k^{sample}$ sums over $L_k$ token-level divergences while $d_k^{decision}$ is a single scalar, yielding $|d_k^{sample}| \approx L_k \cdot \delta$ versus $|d_k^{decision}| \approx \delta$ for token-level divergence magnitude $\delta$. Working backwards: at step $T$, $Q_d^{KL}(s_T, \text{STOP}) = d_T^{decision} \sim O(1)$, while $Q_{sample}^{KL}(s_{T-1}, a'_T) = d_T^{sample} + \gamma \cdot Q_d^{KL}(s_T, \text{STOP}) \approx O(L_T)$. This asymmetry propagates: $Q_{sample}^{KL}$ is dominated by immediate $O(L_k)$ penalties, while $Q_d^{KL}$ has small immediate terms. No unified sufficient statistic $\Phi$ exists since satisfying both $\Phi(s_k) \approx O(L_k) + \gamma\Phi(s_{k+1})$ (sampling) and $\Phi(s_k) \approx O(1) + \gamma\Phi(s_{k+1})$ (decision) is inconsistent when $L_k \gg 1$. See Appendix~\ref{app:proof_thm_kl}. 

\textit{Interpretation.} The KL penalty creates asymmetric regularization: changes to $\pi_{sample}$ incur immediate penalties proportional to $L_k$, while changes to $\pi_d$ incur $O(1)$ penalties. Combined with the surrogate reward's symmetric push (Theorem~3.1, where $Q_{sample}^{reward} = Q_d^{reward} = \gamma^{\sum len(\cdot)} A_i$), the net effect is differential learning: $\pi_d$ receives sustained gradients with minimal KL constraint, while $\pi_{sample}$ is heavily regularized.\footnote{This ``unbalanced attribution'' concerns gradient magnitude distribution, not existence of a unified $V^\pi$---which exists by standard theory. The distinction is whether immediate rewards differ by $O(1)$ (balanced) or $O(L)$ (unbalanced) factors.}

\begin{remark}[Interaction with clipping]
In practice, GRPO clips the importance ratio $\pi_\theta/\pi_{\text{old}}$ at $[1-\varepsilon, 1+\varepsilon]$. This ratio factorizes into sampling and decision components (Appendix~\ref{app:alt_proof}, Eq.~8), and the clipping threshold applies symmetrically to both terms, modifying the magnitude of reward gradients but preserving balanced attribution (Theorem~3.1). The $O(L_k)$ versus $O(1)$ asymmetry, by contrast, is structural: it arises from the token-level summation in the KL penalty and is orthogonal to how the importance ratio is clipped. The unbalanced gradient phenomenon therefore persists under clipped objectives.
\end{remark}

\section{Experiment and Calibration}

The theoretical framework predicts that RL's balanced gradient attribution enables coherent learning of both $\pi_{\text{sample}}$ and $\pi_d$, while SFT's unbalanced attribution leads to miss-meshed gradients that fail to develop an effective decision policy. We design experiments to validate these predictions and to decompose observed performance gains into their constituent policy components.

\subsection{Experimental Setup and Results}

\begin{table*}[t!]
\centering
\caption{\textsc{Model Performance On Arithmetic Tasks}} 
\label{tab:accu_over_tasks}

\renewcommand{\arraystretch}{1.25} 
\resizebox{\textwidth}{!}{%
\begin{tabular}{lccccccc} 
\hline\hline \noalign{\vskip 1mm} 

 & \multicolumn{1}{c}{3$\times$3} & \multicolumn{1}{c}{3$\times$4} & \multicolumn{1}{c}{3$\times$5} & \multicolumn{1}{c}{3$\times$6} & \multicolumn{1}{c}{3$\times$7} & \multicolumn{1}{c}{3$\times$8} & \multicolumn{1}{c}{3$\times$9} \\
 & (1) & (2) & (3) & (4) & (5) & (6) & (7) \\
\hline

\multicolumn{8}{c}{A. Baselines} \\
\hline \noalign{\vskip 1mm}

Base 
 & 74.0 & 34.0 & 5.0 & 1.0 & 4.0 & 3.0 & 1.0 \\
 & (65.4, 82.6) & (24.7, 43.3) & (0.7, 9.3) & (-1.0, 3.0) & (0.2, 7.8) & (-0.3, 6.3) & (-1.0, 3.0) \\
\noalign{\vskip 1ex} 

SFT (no reflection) 
 & 75.0 & 32.0 & 6.0 & 2.0 & 1.0 & 0.0 & 1.0 \\
 & (66.5, 83.5) & (22.9, 41.1) & (1.3, 10.7) & (-0.7, 4.7) & (-1.0, 3.0) & (0.0, 0.0) & (-1.0, 3.0) \\

\midrule 

\multicolumn{8}{c}{B. Reflecting Methods} \\
\hline \noalign{\vskip 1mm}

SFT (reflection) 
 & 96.0 & 92.0 & 94.0 & 49.0 & 4.0 & 1.0 & 0.0 \\
 & (92.2, 99.8) & (86.7, 97.3) & (89.3, 98.7) & (39.2, 58.8) & (0.2, 7.8) & (-1.0, 3.0) & (0.0, 0.0) \\
\noalign{\vskip 1ex}

RL 
 & 98.0 & 91.0 & 92.0 & 90.0 & 72.0 & 53.0 & 34.0 \\
 & (95.3, 100.7) & (85.4, 96.6) & (86.7, 97.3) & (84.1, 95.9) & (63.2, 80.8) & (43.2, 62.8) & (24.7, 43.3) \\
\hline
\end{tabular}%
}

\begin{minipage}{\textwidth} 
\vspace{1.5ex}
\footnotesize
\textsc{Note:} Columns (1)–(7) report accuracy percentages and 95\% confidence intervals ($n=100$) for arithmetic tasks of increasing complexity. Panel A displays results for baselines trained without retry patterns. Panel B displays results for models trained on trajectories with explicit reflection or reinforcement learning. All models are based on \texttt{Qwen2.5-7B-Instruct} \citep{qwen2025qwen25technicalreport}. 
\end{minipage}
\end{table*}

Table~\ref{tab:accu_over_tasks} presents model performance across multiplication tasks of increasing difficulty. The training distribution consists of $4\times 5$ and $5\times 4$ multiplication (Appendix~\ref{app:experiment_details}); all evaluation tasks ($3\times 3$ through $3\times 9$) are out-of-distribution (OOD), ordered by increasing difficulty to trace how each model degrades. Both RL and SFT (reflection) achieve near-ceiling performance on the easiest OOD tasks (96--98\% on $3\times 3$, 91--92\% on $3\times 4$) and remain comparable on $3\times 5$. The critical divergence emerges at $3\times 6$: RL maintains 90\% accuracy while SFT (reflection) drops to 49\%. This gap widens dramatically further OOD---on $3\times 9$, RL achieves 34\% versus SFT (reflection)'s 0\%.

SFT (reflection) substantially outperforms both Base and SFT (no reflection) across all task difficulties, indicating that reflection-rich training data does improve performance. However, the steep degradation pattern suggests this improvement reflects memorization of training-distribution patterns rather than learned generalization. Later in the section, we develop a statistical approach to estimate the performance of $\pi_{\text{sample}}$ and $\pi_d$, decomposing the performance gap between RL and SFT (reflection) into contributions from each policy component.

\subsection{A Simple Calibration of the Model}

To decompose the performance gap between RL and SFT into contributions from $\pi_{\text{sample}}$ and $\pi_d$, we construct a calibration model that abstracts the LLM's behavior into the two-stage decision-sampling process.

\paragraph{Sampling Accuracy ($\pi_{\text{sample}}$).} We model the sampling policy by a single probability $p_s = P(A_k = A_Q^*)$, representing the likelihood that any given sample is correct.

\paragraph{Decision Policy ($\pi_d$).} We model the decision policy as a classifier with two parameters:
\begin{align*}
p_{d|C} &= P(\texttt{STOP} \mid A_k = A_Q^*), \\
p_{d|W} &= P(\texttt{RESAMPLE} \mid A_k \neq A_Q^*).
\end{align*}
The parameter $p_{d|C}$ captures the probability of correctly accepting a correct answer; $p_{d|W}$ captures the probability of correctly rejecting an incorrect answer. An effective decision policy requires both to be high. The framework predicts that RL develops high $p_{d|W}$ through the surrogate reward's balanced gradient attribution\footnote{Also it's reasonable to conjecture that negative sampling and negative reward would tend to develops high $p_{d|C}$.}.

\paragraph{Model Accuracy.} Under the two-stage model, overall accuracy is:
\begin{equation*}
\resizebox{0.8\hsize}{!}{$\text{Model Acc} = \frac{p_s \cdot p_{d|C}}{1 - (p_s \cdot (1 - p_{d|C}) + (1 - p_s) \cdot p_{d|W})}$}
\label{eq:model_acc}
\end{equation*}

We estimate $(p_s, p_{d|C}, p_{d|W})$ from model outputs: $p_s$ from first-attempt accuracy, and the decision parameters from observed stopping and resampling behavior conditional on answer correctness.

\begin{figure*}[t!]
    \centering
    \begin{subfigure}[b]{0.48\textwidth}
        \centering
        \includegraphics[width=\linewidth]{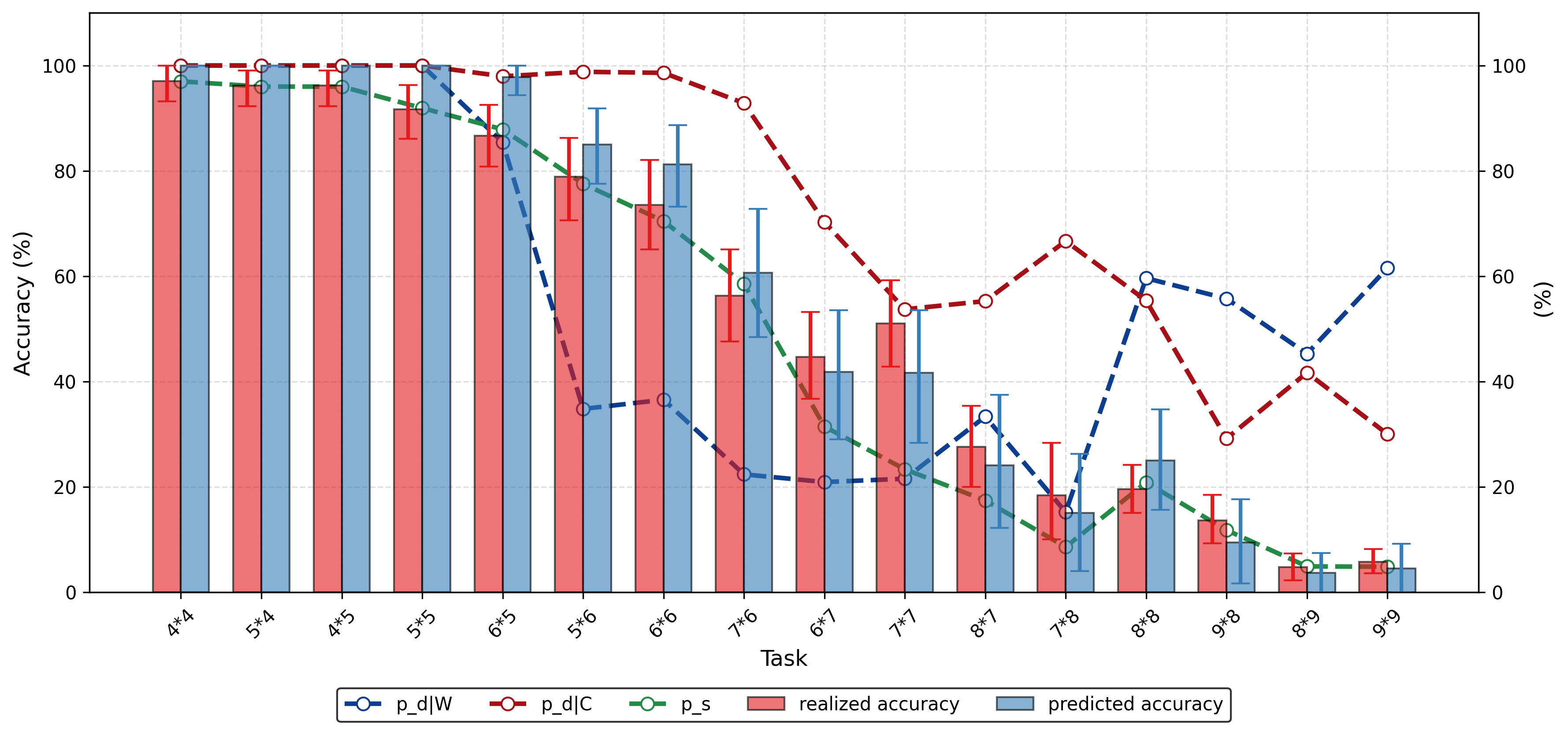}
        \caption{RL Thinking Model}
        \label{fig:calibration_rl}
    \end{subfigure}
    \hfill 
    \begin{subfigure}[b]{0.48\textwidth}
        \centering
        \includegraphics[width=\linewidth]{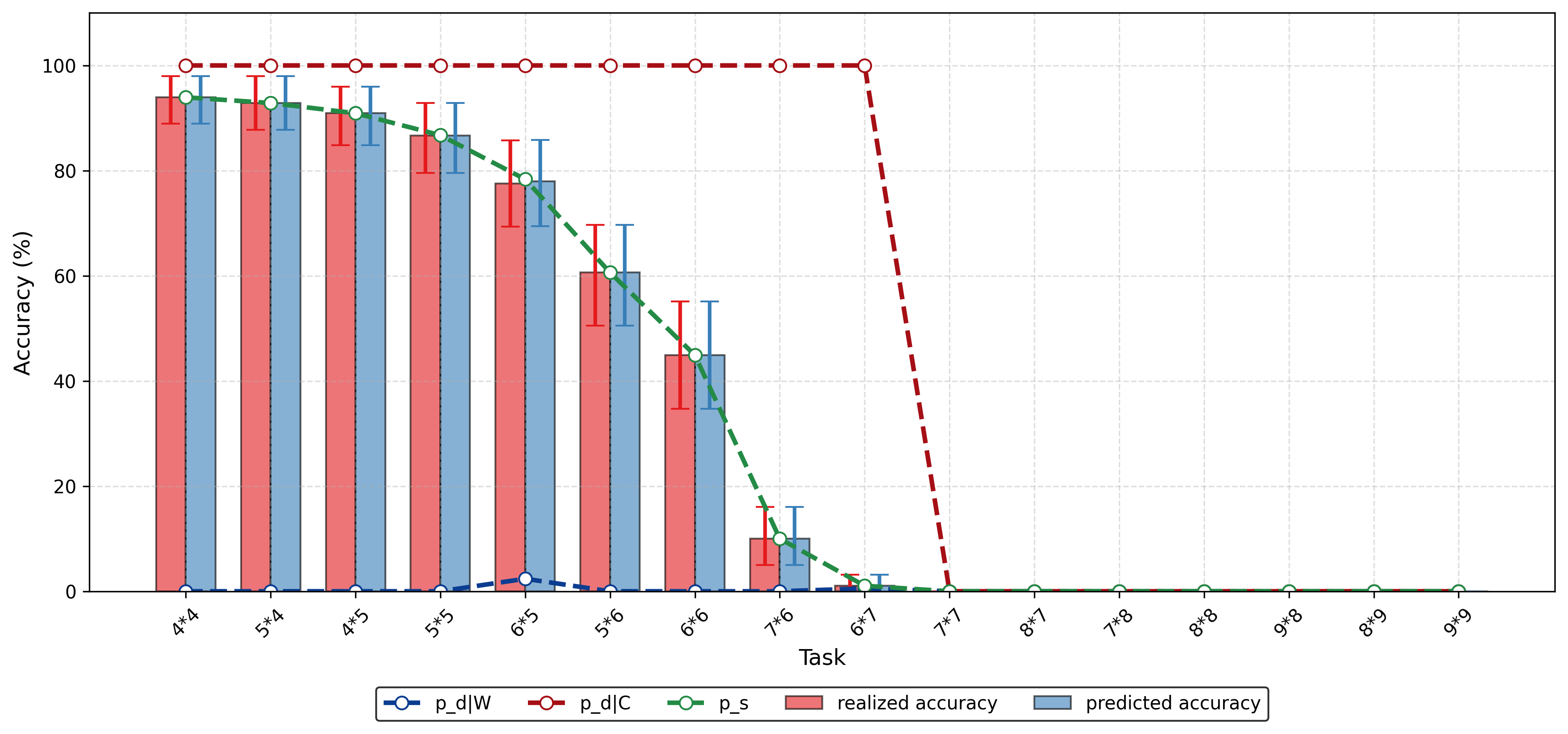}
        \caption{SFT Thinking Model}
        \label{fig:calibration_sft}
    \end{subfigure}
    
    \caption{\textbf{Calibrated Policy Parameters Across Arithmetic Task Sizes.} Comparison of the RL model (a) versus the SFT model (b). \textbf{$p_s$}: sampling accuracy (approximated by 1st try accuracy); \textbf{$p_{d|C}$}: probability of stopping conditional on getting a correct answer; \textbf{$p_{d|W}$}: probability of resampling conditional on getting an incorrect answer. Note that while $p_{d|C}$ remains high for both, the SFT model fails to maintain $p_{d|W}$ on harder tasks.}
    \label{fig:calibration_comparison}
\end{figure*}

Figure~\ref{fig:calibration_rl} presents calibrated parameters for the RL-trained model. The close agreement between predicted and observed accuracy validates the two-stage decomposition. Three patterns emerge: (1) sampling accuracy $p_s$ degrades monotonically with task difficulty, from approximately 80\% on the easiest OOD tasks to below 20\% on the most challenging ones; (2) decision parameters exhibit remarkable stability, with $p_{d|C}$ near ceiling and $p_{d|W}$ sustaining 40~60\% even far out-of-distribution; (3) the calibrated model accurately predicts observed accuracy across the full difficulty range.

Figure~\ref{fig:calibration_sft} reveals a starkly different pattern for SFT. While $p_{d|C}$ similarly remains high, $p_{d|W}$ collapses toward zero on OOD tasks—the model fails to reject incorrect answers, instead echoing first attempts. Sampling accuracy $p_s$ also degrades more precipitously than RL. The critical distinction: RL maintains discriminative ability ($p_{d|W} > 0$) where SFT does not, confirming that RL's superior generalization stems from learning an effective $\pi_d$ that enables error correction even when $\pi_{\text{sample}}$ falters.

\section{Towards Understanding the Insufficiency of SFT, and the Power of RL}

\subsection{SFT Echo Chamber}

Recent systematic analysis reveals that SFT consistently fails to develop self-correction despite extensive efforts to engineer reflection-rich training data. \cite{kang2025trymattersrevisitingrole} constructed SFT datasets with varying reflection steps and corrective $F\rightarrow T$ transition proportions (0\% to 100\%). Models trained on maximum-reflection data outperformed minimal-reflection variants by 4.05\%, yet decomposition reveals first-candidate accuracy accounts for 3.75\% while reflection contributes only 0.3\%. More critically, $p(F\rightarrow T)$ showed no meaningful improvement across any configuration Table~\ref{tab:kang2025}.

\begin{table}[H]
    \centering
    \caption{\cite{kang2025trymattersrevisitingrole}: FT Training Cannot Improve Self-Correction }
    \resizebox{0.9\hsize}{!}{%
    \begin{tabular}{l l l}
        \toprule
        \textbf{F$\to$T Ratio in Data} & \textbf{p(F$\to$T) Llama3.1-8B} & \textbf{p(F$\to$T) Qwen2.5-7B} \\
        \midrule
        100\% & 0.053 & 0.036 \\ [1.5ex]
        50\%  & 0.058 & 0.045 \\ [1.5ex]
        0\%   & 0.050 & 0.041 \\
        \bottomrule
    \end{tabular}}
    \label{tab:kang2025}
    \vspace{1ex}
    \raggedright
    \textit{Note: $p(F \to T)$ remains flat regardless of corrective reflection exposure.}
\end{table}

The gradient attribution framework explains this pattern. SFT operates similarly to a KL-divergence between $\pi_{\theta}$ and $\pi_{data}$, maximizing likelihood of entire trajectories without decomposing credit between "good decision to RESAMPLE" versus "good sampling." Even with abundant $F\rightarrow T$ patterns in training data, SFT teaches $\pi_{sample}$ to produce reflection-like text but provides misattributed gradient signal for $\pi_{d}$ to learn when to trigger correction, creating the 'SFT echo chamber'.

\subsection{Memorization vs. Generalization.}

If SFT primarily improves $\pi_{sample}$ without developing $\pi_d$, it is natural to predict SFT models will memorize training distributions rather than learn generalizable decision rules. \cite{chu2025sftmemorizesrlgeneralizes} test this directly across four task variants. The results are stark: SFT degrades OOD performance by 8–80\% while RL improves it by 3–61\%. Crucially, this pattern persists even when SFT uses sub-optimal trajectories containing errors, the memorization stems from the training objective's structure, not data quality.

\begin{figure}
    \centering
    \includegraphics[width=1\linewidth]{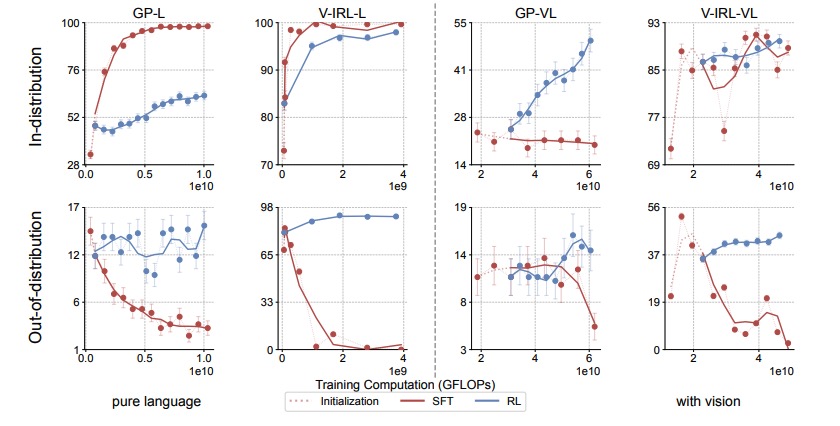}
    \caption{\cite{chu2025sftmemorizesrlgeneralizes}: ID/OOD performance comparison between SFT/RL}
    \label{fig:chuetal-fig5}
\end{figure}

\subsection{Why Does Dynamic Fine-tuning Work Better}

\cite{wu2025generalizationsftreinforcementlearning} prove that SFT's gradient is equivalent to policy gradient with an implicit reward inversely proportional to model confidence ($\frac{1}{\pi_{\theta}}$). This $\frac{1}{\pi_{\theta}}$ weighting creates policy-entangled rewards where Q-values depend on future policy probabilities rather than outcomes, precluding a joint Q-function for $\pi_{sample}$ and $\pi_{d}$\footnote{Refer to Appendix~\ref{app:proofs} and \ref{app:alt_proof} for proof of unbalanced gradient attribution.}. DFT rescales the objective by $\pi_{\theta}$, mitigating this entanglement\footnote{Though not fully canceling it due to length asymmetry; we provide rigorous proof in Appendix~\ref{app:sft_dft}.} so that Q-values depend more on trajectory outcomes. Their empirical gains are substantial: on Qwen2.5-Math-7B, standard SFT degrades AIME24 from 6.68 to 2.48, while DFT improves it to 8.56. We prove that DFT has improved gradient attribution properties in Appendix~\ref{app:sft_dft}.

\section{Conclusion}\label{sec:conclusion}

We develop the Two-Stage Decision-Sampling Hypothesis as a mechanistic framework explaining why RL post-training produces self-reflection while SFT fails. We introduce the gradient attribution property, a novel analytical tool characterizing how reward gradients distribute between multitasking policy components under the DS-Hypothesis. We prove that standard surrogate rewards exhibit balanced gradient attribution, where both policy components receive proportional learning signals through a unified sufficient statistic. In contrast, KL-divergence penalties exhibit imbalanced gradient attribution: length-weighting in token-level objectives creates asymmetric regularization that heavily constrains $\pi_{sample}$ while leaving $\pi_d$ under-optimized. This asymmetry explains the fundamental mechanistic gap between RL and SFT.

Empirical validation on arithmetic reasoning confirms the framework's predictions. The calibrated two-stage decomposition reveals that RL's superior out-of-distribution generalization stems primarily from learning an effective decision policy rather than improved sampling. We further demonstrate how the gradient attribution framework explains multiple empirical phenomena: SFT's echoing effect, memorization versus generalization patterns, DFT's effectiveness. Independent results across language, code, vision, and mathematical domains \citep{chu2025sftmemorizesrlgeneralizes, kang2025trymattersrevisitingrole} are consistent with the framework's domain-agnostic predictions, though direct decomposition in these settings remains for future work.


\section*{Limitations}

Our analysis abstracts away clipping, trust-region constraints, and other stabilization mechanisms standard in practical implementations (e.g., PPO, GRPO) to preserve analytical tractability and result generality. These mechanisms sacrifice analytical properties in favor of training stability. Extending the gradient attribution framework to exactly match specific algorithmic implementations requires additional theoretical work (e.g. Characterizing how clipping interacts with the reward structure or how trust regions modify the effective gradient flow).

We adopt mathematical reasoning as our empirical context for its clarity, well-defined correctness criteria, and controllable difficulty. The theoretical contributions---the gradient attribution properties and the balanced/unbalanced characterization---are derived from the mathematical structure of the training objective and are domain-agnostic. Independent results are consistent with the framework's predictions: \cite{chu2025sftmemorizesrlgeneralizes} demonstrate across four domains (language, code, vision, and math) that SFT degrades OOD performance by 8--80\% while RL improves it by 3--61\%; \cite{kang2025trymattersrevisitingrole} show on MATH benchmarks that $p(F\!\to\!T)$ remains flat for SFT regardless of corrective data composition. Extending the calibration-style decomposition of $\pi_{\text{sample}}$ and $\pi_d$ to these domains remains an important direction for future work (e.g. developing valid operation realizations of decision boundaries in code generation, open-ended reasoning, or creative tasks).

We analyze binary correctness rewards $R(\tau) = \mathbb{I}(A_T = A^*_Q)$. Many practical applications employ more nuanced reward functions, including partial credit, process-based rewards, or learned reward models. The gradient attribution properties under these alternative reward structures may differ and warrant separate investigation.

\section*{Acknowledgments}

AI assistants were used for editing and proofreading text, and LaTeX formatting. All technical content, theoretical derivations, experimental design, and scientific conclusions are the authors' own work. AI-generated content was reviewed and verified by the authors.


\bibliography{main}

@misc{schulman2017trustregionpolicyoptimization,
      title={Trust Region Policy Optimization}, 
      author={John Schulman and Sergey Levine and Philipp Moritz and Michael I. Jordan and Pieter Abbeel},
      year={2017},
      eprint={1502.05477},
      archivePrefix={arXiv},
      primaryClass={cs.LG},
      url={https://arxiv.org/abs/1502.05477}, 
}

@misc{kang2025trymattersrevisitingrole,
      title={First Try Matters: Revisiting the Role of Reflection in Reasoning Models}, 
      author={Liwei Kang and Yue Deng and Yao Xiao and Zhanfeng Mo and Wee Sun Lee and Lidong Bing},
      year={2025},
      eprint={2510.08308},
      archivePrefix={arXiv},
      primaryClass={cs.AI},
      url={https://arxiv.org/abs/2510.08308}, 
}

@article{wei2022chain,
  title={Chain-of-thought prompting elicits reasoning in large language models},
  author={Wei, Jason and Wang, Xuezhi and Schuurmans, Dale and Bosma, Maarten and Ichter, Brian and Xia, Fei and Chi, Ed and Le, Quoc V and Zhou, Denny},
  journal={Advances in Neural Information Processing Systems},
  volume={35},
  pages={24824--24837},
  year={2022}
}

@inproceedings{huang2024llm,
  title={Large language models cannot self-correct reasoning yet},
  author={Huang, Jie and Gu, Xinyun and Shen, Leyang and Kordi, Yeganeh and Wu, Bailin and Zhang, Kaiyan and Li, Dacheng and Li, Qianyu and Zhang, Pengcheng and Liu, Yiming and others},
  booktitle={International Conference on Learning Representations},
  year={2024}
}

@article{kamoi2024critical,
  title={A critical analysis of self-correction capabilities in large language models},
  author={Kamoi, Ryo and Golovneva, Olga and Hsu, Po-Yao and Celikyilmaz, Asli and Kim, Yejin and Choi, Yejin and Freitag, Markus},
  journal={Transactions of the Association for Computational Linguistics},
  volume={12},
  year={2024}
}

@article{kumar2024score,
  title={Training language models to self-correct via reinforcement learning},
  author={Kumar, Aviral and Zhuang, Vincent and Agarwal, Rishabh and Su, Yi and Kazemi, Seyed Mehran and Levine, Sergey},
  journal={arXiv preprint arXiv:2409.12917},
  year={2024}
}

@article{stiennon2020learning,
  title={Learning to summarize with human feedback},
  author={Stiennon, Nisan and Ouyang, Long and Wu, Jeffrey and Ziegler, Daniel and Lowe, Ryan and Voss, Chelsea and Radford, Alec and Amodei, Dario and Christiano, Paul F},
  journal={Advances in Neural Information Processing Systems},
  volume={33},
  pages={3008--3021},
  year={2020}
}

@article{shao2024deepseekmath,
  title={{D}eep{S}eek{M}ath: Pushing the limits of mathematical reasoning in open language models},
  author={Shao, Zhihong and Wang, Peiyi and Zhu, Qihao and Xu, Runxin and Song, Junxiao and Zhang, Mingchuan and Li, YK and Wu, Y and Guo, Daya},
  journal={arXiv preprint arXiv:2402.03300},
  year={2024}
}

@article{deepseek2025r1,
  title={{D}eep{S}eek-{R}1: Incentivizing reasoning capability in {LLM}s via reinforcement learning},
  author={{DeepSeek-AI} and Guo, Daya and others},
  journal={arXiv preprint arXiv:2501.12948},
  year={2025}
}

@article{nogueira2021investigating,
  title={Investigating the limitations of transformers with simple arithmetic tasks},
  author={Nogueira, Rodrigo and Jiang, Zhiying and Lin, Jimmy},
  journal={arXiv preprint arXiv:2102.13019},
  year={2021}
}

@article{anil2022exploring,
  title={Exploring length generalization in large language models},
  author={Anil, Cem and Wu, Yuhuai and Andreassen, Anders and Lewkowycz, Aitor and Misra, Vedant and Ramasesh, Vinay and Slone, Ambrose and Gur-Ari, Guy and Dyer, Ethan and Neyshabur, Behnam},
  journal={Advances in Neural Information Processing Systems},
  volume={35},
  pages={38546--38556},
  year={2022}
}

@article{lee2023teaching,
  title={Teaching arithmetic to small transformers},
  author={Lee, Nayoung and Sreenivasan, Kartik and Lee, Jason D and Lee, Kangwook and Papailiopoulos, Dimitris},
  journal={arXiv preprint arXiv:2307.03381},
  year={2023}
}

@inproceedings{xu2025principled,
  title={Principled understanding of generalization for generative transformer models in arithmetic reasoning tasks},
  author={Xu, Xingcheng and Zhao, Zibo and Zhang, Haipeng and Yang, Yanqing},
  booktitle={Proceedings of the 63rd Annual Meeting of the Association for Computational Linguistics},
  year={2025},
  note={To appear}
}

@article{sutton1999policy,
  title={Policy gradient methods for reinforcement learning with function approximation},
  author={Sutton, Richard S and McAllester, David A and Singh, Satinder P and Mansour, Yishay},
  journal={Advances in Neural Information Processing Systems},
  volume={12},
  year={1999}
}

@article{suttonetal1999,
title = {Between MDPs and semi-MDPs: A framework for temporal abstraction in reinforcement learning},
journal = {Artificial Intelligence},
volume = {112},
number = {1},
pages = {181-211},
year = {1999},
issn = {0004-3702},
doi = {https://doi.org/10.1016/S0004-3702(99)00052-1},
url = {https://www.sciencedirect.com/science/article/pii/S0004370299000521},
author = {Richard S. Sutton and Doina Precup and Satinder Singh}
}

@misc{agarwal2020theorypolicygradientmethods,
      title={On the Theory of Policy Gradient Methods: Optimality, Approximation, and Distribution Shift}, 
      author={Alekh Agarwal and Sham M. Kakade and Jason D. Lee and Gaurav Mahajan},
      year={2020},
      eprint={1908.00261},
      archivePrefix={arXiv},
      primaryClass={cs.LG},
      url={https://arxiv.org/abs/1908.00261}, 
}

@misc{feng2023revealingmysterychainthought,
      title={Towards Revealing the Mystery behind Chain of Thought: A Theoretical Perspective}, 
      author={Guhao Feng and Bohang Zhang and Yuntian Gu and Haotian Ye and Di He and Liwei Wang},
      year={2023},
      eprint={2305.15408},
      archivePrefix={arXiv},
      primaryClass={cs.LG},
      url={https://arxiv.org/abs/2305.15408}, 
}

@misc{xiong2025selfrewardingcorrectionmathematicalreasoning,
      title={Self-rewarding correction for mathematical reasoning}, 
      author={Wei Xiong and Hanning Zhang and Chenlu Ye and Lichang Chen and Nan Jiang and Tong Zhang},
      year={2025},
      eprint={2502.19613},
      archivePrefix={arXiv},
      primaryClass={cs.AI},
      url={https://arxiv.org/abs/2502.19613}, 
}

@misc{jiang2025pagmultiturnreinforcedllm,
      title={PAG: Multi-Turn Reinforced LLM Self-Correction with Policy as Generative Verifier}, 
      author={Yuhua Jiang and Yuwen Xiong and Yufeng Yuan and Chao Xin and Wenyuan Xu and Yu Yue and Qianchuan Zhao and Lin Yan},
      year={2025},
      eprint={2506.10406},
      archivePrefix={arXiv},
      primaryClass={cs.CL},
      url={https://arxiv.org/abs/2506.10406}, 
}

@misc{zhao2025boostingllmreasoningspontaneous,
      title={Boosting LLM Reasoning via Spontaneous Self-Correction}, 
      author={Xutong Zhao and Tengyu Xu and Xuewei Wang and Zhengxing Chen and Di Jin and Liang Tan and Yen-Ting and Zishun Yu and Zhuokai Zhao and Yun He and Sinong Wang and Han Fang and Sarath Chandar and Chen Zhu},
      year={2025},
      eprint={2506.06923},
      archivePrefix={arXiv},
      primaryClass={cs.AI},
      url={https://arxiv.org/abs/2506.06923}, 
}

@misc{ma2025s2rteachingllmsselfverify,
      title={S$^2$R: Teaching LLMs to Self-verify and Self-correct via Reinforcement Learning}, 
      author={Ruotian Ma and Peisong Wang and Cheng Liu and Xingyan Liu and Jiaqi Chen and Bang Zhang and Xin Zhou and Nan Du and Jia Li},
      year={2025},
      eprint={2502.12853},
      archivePrefix={arXiv},
      primaryClass={cs.CL},
      url={https://arxiv.org/abs/2502.12853}, 
}

@misc{ouyang2022traininglanguagemodelsfollow,
      title={Training language models to follow instructions with human feedback}, 
      author={Long Ouyang and Jeff Wu and Xu Jiang and Diogo Almeida and Carroll L. Wainwright and Pamela Mishkin and Chong Zhang and Sandhini Agarwal and Katarina Slama and Alex Ray and John Schulman and Jacob Hilton and Fraser Kelton and Luke Miller and Maddie Simens and Amanda Askell and Peter Welinder and Paul Christiano and Jan Leike and Ryan Lowe},
      year={2022},
      eprint={2203.02155},
      archivePrefix={arXiv},
      primaryClass={cs.CL},
      url={https://arxiv.org/abs/2203.02155}, 
}

@misc{schulman2017proximalpolicyoptimizationalgorithms,
      title={Proximal Policy Optimization Algorithms}, 
      author={John Schulman and Filip Wolski and Prafulla Dhariwal and Alec Radford and Oleg Klimov},
      year={2017},
      eprint={1707.06347},
      archivePrefix={arXiv},
      primaryClass={cs.LG},
      url={https://arxiv.org/abs/1707.06347}, 
}

@misc{rafailov2024directpreferenceoptimizationlanguage,
      title={Direct Preference Optimization: Your Language Model is Secretly a Reward Model}, 
      author={Rafael Rafailov and Archit Sharma and Eric Mitchell and Stefano Ermon and Christopher D. Manning and Chelsea Finn},
      year={2024},
      eprint={2305.18290},
      archivePrefix={arXiv},
      primaryClass={cs.LG},
      url={https://arxiv.org/abs/2305.18290}, 
}

@misc{cobbe2021trainingverifierssolvemath,
      title={Training Verifiers to Solve Math Word Problems}, 
      author={Karl Cobbe and Vineet Kosaraju and Mohammad Bavarian and Mark Chen and Heewoo Jun and Lukasz Kaiser and Matthias Plappert and Jerry Tworek and Jacob Hilton and Reiichiro Nakano and Christopher Hesse and John Schulman},
      year={2021},
      eprint={2110.14168},
      archivePrefix={arXiv},
      primaryClass={cs.LG},
      url={https://arxiv.org/abs/2110.14168}, 
}

@misc{hendrycks2021measuringmathematicalproblemsolving,
      title={Measuring Mathematical Problem Solving With the MATH Dataset}, 
      author={Dan Hendrycks and Collin Burns and Saurav Kadavath and Akul Arora and Steven Basart and Eric Tang and Dawn Song and Jacob Steinhardt},
      year={2021},
      eprint={2103.03874},
      archivePrefix={arXiv},
      primaryClass={cs.LG},
      url={https://arxiv.org/abs/2103.03874}, 
}

@misc{yu2024metamathbootstrapmathematicalquestions,
      title={MetaMath: Bootstrap Your Own Mathematical Questions for Large Language Models}, 
      author={Longhui Yu and Weisen Jiang and Han Shi and Jincheng Yu and Zhengying Liu and Yu Zhang and James T. Kwok and Zhenguo Li and Adrian Weller and Weiyang Liu},
      year={2024},
      eprint={2309.12284},
      archivePrefix={arXiv},
      primaryClass={cs.CL},
      url={https://arxiv.org/abs/2309.12284}, 
}

@misc{openai2024openaio1card,
      title={OpenAI o1 System Card}, 
      author={OpenAI and : and Aaron Jaech and Adam Kalai and Adam Lerer and Adam Richardson and Ahmed El-Kishky and Aiden Low and Alec Helyar and Aleksander Madry and Alex Beutel and Alex Carney and Alex Iftimie and Alex Karpenko and Alex Tachard Passos and Alexander Neitz and Alexander Prokofiev and Alexander Wei and Allison Tam and Ally Bennett and Ananya Kumar and Andre Saraiva and Andrea Vallone and Andrew Duberstein and Andrew Kondrich and Andrey Mishchenko and Andy Applebaum and Angela Jiang and Ashvin Nair and Barret Zoph and Behrooz Ghorbani and Ben Rossen and Benjamin Sokolowsky and Boaz Barak and Bob McGrew and Borys Minaiev and Botao Hao and Bowen Baker and Brandon Houghton and Brandon McKinzie and Brydon Eastman and Camillo Lugaresi and Cary Bassin and Cary Hudson and Chak Ming Li and Charles de Bourcy and Chelsea Voss and Chen Shen and Chong Zhang and Chris Koch and Chris Orsinger and Christopher Hesse and Claudia Fischer and Clive Chan and Dan Roberts and Daniel Kappler and Daniel Levy and Daniel Selsam and David Dohan and David Farhi and David Mely and David Robinson and Dimitris Tsipras and Doug Li and Dragos Oprica and Eben Freeman and Eddie Zhang and Edmund Wong and Elizabeth Proehl and Enoch Cheung and Eric Mitchell and Eric Wallace and Erik Ritter and Evan Mays and Fan Wang and Felipe Petroski Such and Filippo Raso and Florencia Leoni and Foivos Tsimpourlas and Francis Song and Fred von Lohmann and Freddie Sulit and Geoff Salmon and Giambattista Parascandolo and Gildas Chabot and Grace Zhao and Greg Brockman and Guillaume Leclerc and Hadi Salman and Haiming Bao and Hao Sheng and Hart Andrin and Hessam Bagherinezhad and Hongyu Ren and Hunter Lightman and Hyung Won Chung and Ian Kivlichan and Ian O'Connell and Ian Osband and Ignasi Clavera Gilaberte and Ilge Akkaya and Ilya Kostrikov and Ilya Sutskever and Irina Kofman and Jakub Pachocki and James Lennon and Jason Wei and Jean Harb and Jerry Twore and Jiacheng Feng and Jiahui Yu and Jiayi Weng and Jie Tang and Jieqi Yu and Joaquin Quiñonero Candela and Joe Palermo and Joel Parish and Johannes Heidecke and John Hallman and John Rizzo and Jonathan Gordon and Jonathan Uesato and Jonathan Ward and Joost Huizinga and Julie Wang and Kai Chen and Kai Xiao and Karan Singhal and Karina Nguyen and Karl Cobbe and Katy Shi and Kayla Wood and Kendra Rimbach and Keren Gu-Lemberg and Kevin Liu and Kevin Lu and Kevin Stone and Kevin Yu and Lama Ahmad and Lauren Yang and Leo Liu and Leon Maksin and Leyton Ho and Liam Fedus and Lilian Weng and Linden Li and Lindsay McCallum and Lindsey Held and Lorenz Kuhn and Lukas Kondraciuk and Lukasz Kaiser and Luke Metz and Madelaine Boyd and Maja Trebacz and Manas Joglekar and Mark Chen and Marko Tintor and Mason Meyer and Matt Jones and Matt Kaufer and Max Schwarzer and Meghan Shah and Mehmet Yatbaz and Melody Y. Guan and Mengyuan Xu and Mengyuan Yan and Mia Glaese and Mianna Chen and Michael Lampe and Michael Malek and Michele Wang and Michelle Fradin and Mike McClay and Mikhail Pavlov and Miles Wang and Mingxuan Wang and Mira Murati and Mo Bavarian and Mostafa Rohaninejad and Nat McAleese and Neil Chowdhury and Neil Chowdhury and Nick Ryder and Nikolas Tezak and Noam Brown and Ofir Nachum and Oleg Boiko and Oleg Murk and Olivia Watkins and Patrick Chao and Paul Ashbourne and Pavel Izmailov and Peter Zhokhov and Rachel Dias and Rahul Arora and Randall Lin and Rapha Gontijo Lopes and Raz Gaon and Reah Miyara and Reimar Leike and Renny Hwang and Rhythm Garg and Robin Brown and Roshan James and Rui Shu and Ryan Cheu and Ryan Greene and Saachi Jain and Sam Altman and Sam Toizer and Sam Toyer and Samuel Miserendino and Sandhini Agarwal and Santiago Hernandez and Sasha Baker and Scott McKinney and Scottie Yan and Shengjia Zhao and Shengli Hu and Shibani Santurkar and Shraman Ray Chaudhuri and Shuyuan Zhang and Siyuan Fu and Spencer Papay and Steph Lin and Suchir Balaji and Suvansh Sanjeev and Szymon Sidor and Tal Broda and Aidan Clark and Tao Wang and Taylor Gordon and Ted Sanders and Tejal Patwardhan and Thibault Sottiaux and Thomas Degry and Thomas Dimson and Tianhao Zheng and Timur Garipov and Tom Stasi and Trapit Bansal and Trevor Creech and Troy Peterson and Tyna Eloundou and Valerie Qi and Vineet Kosaraju and Vinnie Monaco and Vitchyr Pong and Vlad Fomenko and Weiyi Zheng and Wenda Zhou and Wes McCabe and Wojciech Zaremba and Yann Dubois and Yinghai Lu and Yining Chen and Young Cha and Yu Bai and Yuchen He and Yuchen Zhang and Yunyun Wang and Zheng Shao and Zhuohan Li},
      year={2024},
      eprint={2412.16720},
      archivePrefix={arXiv},
      primaryClass={cs.AI},
      url={https://arxiv.org/abs/2412.16720}, 
}

@misc{bercovich2025llamanemotronefficientreasoningmodels,
      title={Llama-Nemotron: Efficient Reasoning Models}, 
      author={Akhiad Bercovich and Itay Levy and Izik Golan and Mohammad Dabbah and Ran El-Yaniv and Omri Puny and Ido Galil and Zach Moshe and Tomer Ronen and Najeeb Nabwani and Ido Shahaf and Oren Tropp and Ehud Karpas and Ran Zilberstein and Jiaqi Zeng and Soumye Singhal and Alexander Bukharin and Yian Zhang and Tugrul Konuk and Gerald Shen and Ameya Sunil Mahabaleshwarkar and Bilal Kartal and Yoshi Suhara and Olivier Delalleau and Zijia Chen and Zhilin Wang and David Mosallanezhad and Adi Renduchintala and Haifeng Qian and Dima Rekesh and Fei Jia and Somshubra Majumdar and Vahid Noroozi and Wasi Uddin Ahmad and Sean Narenthiran and Aleksander Ficek and Mehrzad Samadi and Jocelyn Huang and Siddhartha Jain and Igor Gitman and Ivan Moshkov and Wei Du and Shubham Toshniwal and George Armstrong and Branislav Kisacanin and Matvei Novikov and Daria Gitman and Evelina Bakhturina and Prasoon Varshney and Makesh Narsimhan and Jane Polak Scowcroft and John Kamalu and Dan Su and Kezhi Kong and Markus Kliegl and Rabeeh Karimi Mahabadi and Ying Lin and Sanjeev Satheesh and Jupinder Parmar and Pritam Gundecha and Brandon Norick and Joseph Jennings and Shrimai Prabhumoye and Syeda Nahida Akter and Mostofa Patwary and Abhinav Khattar and Deepak Narayanan and Roger Waleffe and Jimmy Zhang and Bor-Yiing Su and Guyue Huang and Terry Kong and Parth Chadha and Sahil Jain and Christine Harvey and Elad Segal and Jining Huang and Sergey Kashirsky and Robert McQueen and Izzy Putterman and George Lam and Arun Venkatesan and Sherry Wu and Vinh Nguyen and Manoj Kilaru and Andrew Wang and Anna Warno and Abhilash Somasamudramath and Sandip Bhaskar and Maka Dong and Nave Assaf and Shahar Mor and Omer Ullman Argov and Scot Junkin and Oleksandr Romanenko and Pedro Larroy and Monika Katariya and Marco Rovinelli and Viji Balas and Nicholas Edelman and Anahita Bhiwandiwalla and Muthu Subramaniam and Smita Ithape and Karthik Ramamoorthy and Yuting Wu and Suguna Varshini Velury and Omri Almog and Joyjit Daw and Denys Fridman and Erick Galinkin and Michael Evans and Shaona Ghosh and Katherine Luna and Leon Derczynski and Nikki Pope and Eileen Long and Seth Schneider and Guillermo Siman and Tomasz Grzegorzek and Pablo Ribalta and Monika Katariya and Chris Alexiuk and Joey Conway and Trisha Saar and Ann Guan and Krzysztof Pawelec and Shyamala Prayaga and Oleksii Kuchaiev and Boris Ginsburg and Oluwatobi Olabiyi and Kari Briski and Jonathan Cohen and Bryan Catanzaro and Jonah Alben and Yonatan Geifman and Eric Chung},
      year={2025},
      eprint={2505.00949},
      archivePrefix={arXiv},
      primaryClass={cs.CL},
      url={https://arxiv.org/abs/2505.00949}, 
}

@misc{zhao2024marcoo1openreasoningmodels,
      title={Marco-o1: Towards Open Reasoning Models for Open-Ended Solutions}, 
      author={Yu Zhao and Huifeng Yin and Bo Zeng and Hao Wang and Tianqi Shi and Chenyang Lyu and Longyue Wang and Weihua Luo and Kaifu Zhang},
      year={2024},
      eprint={2411.14405},
      archivePrefix={arXiv},
      primaryClass={cs.CL},
      url={https://arxiv.org/abs/2411.14405}, 
}

@misc{chu2025sftmemorizesrlgeneralizes,
      title={SFT Memorizes, RL Generalizes: A Comparative Study of Foundation Model Post-training}, 
      author={Tianzhe Chu and Yuexiang Zhai and Jihan Yang and Shengbang Tong and Saining Xie and Dale Schuurmans and Quoc V. Le and Sergey Levine and Yi Ma},
      year={2025},
      eprint={2501.17161},
      archivePrefix={arXiv},
      primaryClass={cs.AI},
      url={https://arxiv.org/abs/2501.17161}, 
}

@misc{wu2025generalizationsftreinforcementlearning,
      title={On the Generalization of SFT: A Reinforcement Learning Perspective with Reward Rectification}, 
      author={Yongliang Wu and Yizhou Zhou and Zhou Ziheng and Yingzhe Peng and Xinyu Ye and Xinting Hu and Wenbo Zhu and Lu Qi and Ming-Hsuan Yang and Xu Yang},
      year={2025},
      eprint={2508.05629},
      archivePrefix={arXiv},
      primaryClass={cs.LG},
      url={https://arxiv.org/abs/2508.05629}, 
}

@misc{qwen2025qwen25technicalreport,
      title={Qwen2.5 Technical Report}, 
      author={Qwen and : and An Yang and Baosong Yang and Beichen Zhang and Binyuan Hui and Bo Zheng and Bowen Yu and Chengyuan Li and Dayiheng Liu and Fei Huang and Haoran Wei and Huan Lin and Jian Yang and Jianhong Tu and Jianwei Zhang and Jianxin Yang and Jiaxi Yang and Jingren Zhou and Junyang Lin and Kai Dang and Keming Lu and Keqin Bao and Kexin Yang and Le Yu and Mei Li and Mingfeng Xue and Pei Zhang and Qin Zhu and Rui Men and Runji Lin and Tianhao Li and Tianyi Tang and Tingyu Xia and Xingzhang Ren and Xuancheng Ren and Yang Fan and Yang Su and Yichang Zhang and Yu Wan and Yuqiong Liu and Zeyu Cui and Zhenru Zhang and Zihan Qiu},
      year={2025},
      eprint={2412.15115},
      archivePrefix={arXiv},
      primaryClass={cs.CL},
      url={https://arxiv.org/abs/2412.15115}, 
}

@article{sheng2024hybridflow,
  title   = {HybridFlow: A Flexible and Efficient RLHF Framework},
  author  = {Guangming Sheng and Chi Zhang and Zilingfeng Ye and Xibin Wu and Wang Zhang and Ru Zhang and Yanghua Peng and Haibin Lin and Chuan Wu},
  year    = {2024},
  journal = {arXiv preprint arXiv: 2409.19256}
}

@inproceedings{zheng2024llamafactory,
  title={LlamaFactory: Unified Efficient Fine-Tuning of 100+ Language Models},
  author={Yaowei Zheng and Richong Zhang and Junhao Zhang and Yanhan Ye and Zheyan Luo and Zhangchi Feng and Yongqiang Ma},
  booktitle={Proceedings of the 62nd Annual Meeting of the Association for Computational Linguistics (Volume 3: System Demonstrations)},
  address={Bangkok, Thailand},
  publisher={Association for Computational Linguistics},
  year={2024},
  url={http://arxiv.org/abs/2403.13372}
}


\clearpage
\appendix

\section{Proofs for Gradient Attribution Property Framework}\label{app:proofs}
This appendix provides complete mathematical derivations for all results in Section 3.

\subsection{Proof of Lemma 3.1 (Trajectory Factorization)}

\noindent \textbf{Proof.} We prove this by applying the chain rule of probability to the sequential generation process. A trajectory $\tau$ of length $T$ is:
\begin{equation*}
\resizebox{0.95\hsize}{!}{$
\tau = (A_1, T_1, a_1, A_2, T_2, a_2, \ldots, A_T, T_T, a_T)
$}
\end{equation*}
where $a_k \in \{\text{RESAMPLE}, \text{STOP}\}$ for $k = 1, \ldots, T-1$ and $a_T = \text{STOP}$ (by definition, the trajectory ends with STOP).

By the chain rule:
\begin{equation*}
\begin{split}
P(\tau | Q; \theta) &= P(A_1, T_1 | Q; \theta) \\
&\quad \cdot P(a_1 | A_1, T_1, Q; \theta) \\
&\quad \cdot P(A_2, T_2 | a_1, A_1, T_1, Q; \theta) \cdots
\end{split}
\end{equation*}

Sequentially we write it to be:
\begin{align*}
P(\tau | Q; \theta) &= \pi_{\text{sample}}(A_1, T_1 | s_0; \theta) \nonumber \\
&\quad \cdot \pi_d(\text{RESAMPLE} | s_1; \theta) \nonumber \\
&\quad \cdot \pi_{\text{sample}}(A_2, T_2 | s_1; \theta) \cdots \nonumber \\
&\quad \cdot \pi_d(\text{RESAMPLE} | s_{T-1}; \theta) \nonumber \\
&\quad \cdot \pi_{\text{sample}}(A_T, T_T | s_{T-1}; \theta) \nonumber \\
&\quad \cdot \pi_d(\text{STOP} | s_T; \theta)
\end{align*}

Thus:
\begin{align*}
P(\tau | Q; \theta) &= \left[\prod_{k=1}^T \pi_{\text{sample}}(A_k, T_k | s_{k-1}; \theta)\right] \\
&\quad \cdot \left[\prod_{k=1}^{T-1} \pi_d(\text{R} | s_k; \theta)\right] \cdot \pi_d(\text{S} | s_T; \theta)
\end{align*}
\hfill $\Box$

\subsection{Proof of Corollary 3.1 (Gradient Decomposition)}\label{app:proof_corollary}

\noindent \textbf{Proof.} From Lemma 3.1:
\begin{equation*}
\resizebox{1.0\hsize}{!}{$
\begin{split}
\log P(\tau | Q; \theta) &= \sum_{k=1}^T \log \pi_{\text{sample}}(A_k, T_k | s_{k-1}; \theta) \\
&\quad + \sum_{k=1}^{T-1} \log \pi_d(\text{RESAMPLE} | s_k; \theta) \\
&\quad + \log \pi_d(\text{STOP} | s_T; \theta)
\end{split}
$}
\end{equation*}

Taking the gradient with respect to $\theta$:
\begin{equation*}
\resizebox{1.0\hsize}{!}{$
\begin{split}
\nabla_\theta \log P(\tau | Q; \theta) &= \sum_{k=1}^T \nabla_\theta \log \pi_{\text{sample}}(A_k, T_k | s_{k-1}; \theta) \\
&\quad + \sum_{k=1}^{T-1} \nabla_\theta \log \pi_d(\text{RESAMPLE} | s_k; \theta) \\
&\quad + \nabla_\theta \log \pi_d(\text{STOP} | s_T; \theta)
\end{split}
$}
\end{equation*}

Rewrite the decision gradient more compactly by defining:
\begin{itemize}
    \item $a_k = \text{RESAMPLE}$ for $k = 1, \ldots, T-1$
    \item $a_T = \text{STOP}$
    \item $a_0 = \text{START}$ (initial action, could be included for notational completeness)
\end{itemize}
Then:
\begin{equation*}
\resizebox{1.0\hsize}{!}{$
\begin{aligned}
\nabla_\theta \log P(\tau | Q; \theta) &= \sum_{k=1}^T \nabla_\theta \log \pi_{\text{sample}}(A_k, T_k | s_{k-1}; \theta) \\
&\quad + \sum_{k=0}^T \nabla_\theta \log \pi_d(a_k | s_k; \theta)
\end{aligned}
$}
\end{equation*}

where we interpret $\pi_d(a_0 | s_0) = 1$ (deterministic start) and thus $\nabla_\theta \log \pi_d(a_0 | s_0) = 0$.
This completes the proof. \hfill $\Box$

\subsection{Proof of Lemma 3.2 (Policy Gradient Theorem - Applied)}\label{app:proof_pg}

\textbf{Proof.} We adapt the standard policy gradient theorem to our factorized policy structure.

\textbf{Standard Policy Gradient Theorem}
The standard policy gradient theorem \citep{suttonetal1999} states that for a policy $\pi_\theta$ and expected return $J(\theta) = \mathbb{E}_{\tau \sim \pi\theta}[R(\tau)]$:
\[ \nabla_\theta J(\theta) = \int_{\mathcal{S}} \rho^\pi(s) \int_{\mathcal{A}} \nabla_\theta \pi(a|s;\theta) \cdot Q^\pi(s,a) \, da \, ds \]
where $\rho^\pi(s)$ is the discounted state visitation measure.

\subsubsection{State Visitation Distribution in Our Setting}
Define $\rho^\pi_t(s)$ as the probability of visiting state $s = (Q, A, T)$ at step $t$:
\begin{equation*}
\resizebox{1.0\hsize}{!}{$
\begin{aligned}
\rho^\pi_t(s_t) &= \sum_{s_{t-1}} \rho^\pi_{t-1}(s_{t-1}) \\
&\quad \cdot \pi_d(\text{RESAMPLE} | s_{t-1}) \cdot \pi_{\text{sample}}(A_t, T_t | s_{t-1})
\end{aligned}
$}
\end{equation*}
This says: to reach state $s_t$ at time $t$, we must have been at some state $s_{t-1}$ at time $t-1$, chosen to RESAMPLE, and then sampled $(A_t, T_t)$.

\textbf{Discounted state visitation:}
\[ \rho^\pi(s) = \sum_{t=1}^{\infty} \gamma^t \rho^\pi_t(s) \]

At each state $s_k = (Q, A_k, T_k)$, there are two types of actions:
\begin{itemize}
    \item Sampling actions $a' = (A, T)$ drawn from $\pi_{\text{sample}}(\cdot | s_{k-1})$
    \item Decision actions $a'' \in \{\text{STOP}, \text{RESAMPLE}\}$ drawn from $\pi_d(\cdot | s_k)$
\end{itemize}
However, these occur at different points in the trajectory:
\begin{itemize}
    \item After state $s_{k-1}$, we take sampling action $a'_k = (A_k, T_k)$ to reach state $s_k$
    \item At state $s_k$, we take decision action $a''_k \in \{\text{STOP}, \text{RESAMPLE}\}$
\end{itemize}

For sampling actions at state $s_{k-1}$:
\begin{equation*}
\resizebox{1.0\hsize}{!}{$
\begin{aligned}
& \int_{\mathcal{A}_{\text{sample}}} \nabla_\theta \pi_{\text{sample},\theta}(a' | s_{k-1}) \cdot Q^\pi(s_{k-1}, a') \, da' \\
&\quad = \mathbb{E}_{a' \sim \pi_{\text{sample}}}\left[\nabla_\theta \log \pi_{\text{sample},\theta}(a' | s_{k-1}) \cdot Q^\pi_{\text{sample}}(s_{k-1}, a')\right]
\end{aligned}
$}
\end{equation*}
For decision actions at state $s_k$:
\begin{equation*}
\resizebox{0.8\hsize}{!}{$
\begin{aligned}
& \sum_{a'' \in \{\text{STOP}, \text{RESAMPLE}\}} \nabla_\theta \pi_{d,\theta}(a'' | s_k) \cdot Q^\pi(s_k, a'') \\
&\qquad\qquad = \mathbb{E}_{a'' \sim \pi_d}\left[\nabla_\theta \log \pi_{d,\theta}(a'' | s_k) \cdot Q^\pi_d(s_k, a'')\right]
\end{aligned}
$}
\end{equation*}

Combining over all states in a trajectory of length $T$:
\begin{equation*}
\resizebox{1.0\hsize}{!}{$
\begin{aligned}
\nabla_\theta J(\theta) &= \mathbb{E}_{\tau \sim \pi_\theta}\left[\sum_{t=1}^T \nabla_\theta \log \pi_{\text{sample},\theta}(a'_t | s_{t-1}) \cdot Q^\pi_{\text{sample}}(s_{t-1}, a'_t)\right] \\
&\quad + \mathbb{E}_{\tau \sim \pi_\theta}\left[\sum_{t=1}^T \nabla_\theta \log \pi_{d,\theta}(a''_t | s_t) \cdot Q^\pi_d(s_t, a''_t)\right]
\end{aligned}
$}
\end{equation*}

We can combine these into a single expectation:
\begin{equation*}
\resizebox{1.0\hsize}{!}{$
\begin{aligned}
\nabla_\theta J(\theta) = \mathbb{E}_{\tau \sim \pi_\theta}\Bigg[\sum_{t=0}^T \bigg( & \nabla_\theta \log\pi_{\text{sample}}(a'_t|s_{t-1}) \cdot Q^\pi_{\text{sample}}(s_{t-1}, a'_t) \\
& + \nabla_\theta \log\pi_d(a''_t|s_t) \cdot Q^\pi_d(s_t, a''_t)\bigg)\Bigg]
\end{aligned}
$}
\end{equation*}
where we use the convention that terms with $t=0$ correspond to the initial action. \hfill $\Box$

\subsection{Proof of Theorem 3.1 (Balanced Attribution of Surrogate Reward)}\label{app:proof_thm_balanced}

The surrogate reward for trajectory $\tau_i$ is:
\[ L_{\text{reward}}(\theta) = \mathbb{E}_{\tau_i \sim \pi_{\text{old}}}\left[\frac{\pi_\theta(\tau_i|Q_i)}{\pi_{\text{old}}(\tau_i|Q_i)}A_i\right] \]
where:
\begin{itemize}
    \item $A_i$ is the advantage for trajectory $\tau_i$, computed using Group Relative Advantage Estimation (GRAE)
    \item For query $Q_i$ with multiple sampled trajectories $\{\tau_1, \ldots, \tau_G\}$ (the "group"):
    \[ A_i = \frac{R(\tau_i) - \text{mean}(R(\tau_1), \ldots, R(\tau_G))}{\text{std}(R(\tau_1), \ldots, R(\tau_G))} \]
    \item The reward function is $R(\tau) = \mathbb{I}(A_T = A^*_Q)$
\end{itemize}

Taking the gradient with respect to $\theta$:
\[ \nabla_\theta L_{\text{reward}}(\theta) = \mathbb{E}_{\tau_i \sim \pi_{\text{old}}}\left[\nabla_\theta \left(\frac{\pi_\theta(\tau_i|Q_i)}{\pi_{\text{old}}(\tau_i|Q_i)}\right) \cdot A_i\right] \]
Using the log-derivative trick:
\[ \nabla_\theta \left(\frac{\pi_\theta(\tau_i|Q_i)}{\pi_{\text{old}}(\tau_i|Q_i)}\right) = \frac{\pi_\theta(\tau_i|Q_i)}{\pi_{\text{old}}(\tau_i|Q_i)} \cdot \nabla_\theta \log \pi_\theta(\tau_i|Q_i) \]

From Corollary 3.1:
\begin{equation*}
\resizebox{1.0\hsize}{!}{$
\begin{aligned}
\nabla_\theta \log \pi_\theta(\tau_i|Q_i) &= \sum_{k=1}^{T_i} \nabla_\theta \log \pi_{\text{sample},\theta}(A_k, T_k | s_{k-1}) \\
&\quad + \sum_{k=0}^{T_i} \nabla_\theta \log \pi_{d,\theta}(a_k | s_k)
\end{aligned}
$}
\end{equation*}

Therefore, when sampling from $\pi_{\text{old}}$ and using importance weighting:
\begin{equation*}
\resizebox{1.0\hsize}{!}{$
\begin{aligned}
\nabla_\theta L_{\text{reward}}(\theta) &= \mathbb{E}_{\tau_i \sim \pi_{\text{old}}}\Bigg[\frac{\pi_\theta(\tau_i|Q_i)}{\pi_{\text{old}}(\tau_i|Q_i)} \cdot A_i \\
&\quad \cdot \left(\sum_{k=1}^{T_i} \nabla_\theta \log \pi_{\text{sample},\theta}(\cdot) + \sum_{k=0}^{T_i} \nabla_\theta \log \pi_{d,\theta}(\cdot)\right)\Bigg]
\end{aligned}
$}
\end{equation*}

By importance sampling, this equals:
\begin{equation*}
\resizebox{1.0\hsize}{!}{$
\begin{aligned}
\nabla_\theta L_{\text{reward}}(\theta) &= \mathbb{E}_{\tau_i \sim \pi_\theta}\Bigg[A_i \Bigg(\sum_{k=1}^{T_i} \nabla_\theta \log \pi_{\text{sample},\theta}(A_k, T_k | s_{k-1}) \\
&\quad + \sum_{k=0}^{T_i} \nabla_\theta \log \pi_{d,\theta}(a_k | s_k)\Bigg)\Bigg]
\end{aligned}
$}
\end{equation*}

Comparing with the policy gradient decomposition from Lemma 3.2, we identify the Q-values by noting that the gradient must have the form:
\begin{equation*}
\resizebox{1.0\hsize}{!}{$
\begin{aligned}
\nabla_\theta L_{\text{reward}}(\theta) &= \mathbb{E}_{\tau_i \sim \pi_\theta}\Bigg[\sum_{k=1}^{T_i} \nabla_\theta \log \pi_{\text{sample},\theta}(\cdot) \cdot Q^{\pi,\text{reward}}_{\text{sample}}(\cdot) \\
&\quad + \sum_{k=0}^{T_i} \nabla_\theta \log \pi_{d,\theta}(\cdot) \cdot Q^{\pi,\text{reward}}_d(\cdot)\Bigg]
\end{aligned}
$}
\end{equation*}

For the surrogate reward, both Q-values equal the length-discounted advantage:
\[ Q^{\pi,\text{reward}}_{\text{sample}}(s_{k-1}, (A_k, T_k)) = \gamma^{\sum_{j=k}^{T_i} \text{len}(A_j, T_j)} A_i \]
\[ Q^{\pi,\text{reward}}_d(s_k, a_k) = \gamma^{\sum_{j=k}^{T_i} \text{len}(A_j, T_j)} A_i \]

Define:
\[ \Phi(s_k) = \gamma^{\sum_{j=k}^{T_i} \text{len}(A_j, T_j)} A_i \]
This is the sufficient statistic of future rewards at state $s_k$. It encodes: "the trajectory will receive advantage $A_i$ at the end, which is at temporal distance $\sum_{j=k}^{T_i} \text{len}(A_j, T_j)$ from now."
Then:
\[ Q^{\pi,\text{reward}}_{\text{sample}}(s_{k-1}, a'_k) = \Phi(s_k) \]
\[ Q^{\pi,\text{reward}}_d(s_k, a''_k) = \Phi(s_k) \]
Both Q-values can be expressed using the same sufficient statistic $\Phi$, evaluated at the appropriate state. This satisfies Definition 3.1 for Balanced Gradient Attribution. $\Box$

\subsection{Proof of Theorem 3.2 (Unbalanced Attribution of KL Penalty)}\label{app:proof_thm_kl}

We analyze the gradient attribution properties of the token-level KL penalty.

\paragraph{KL Penalty Formulation.}
The standard token-level KL penalty, computed on trajectories sampled from the current policy, is:
\begin{equation*}
    D_{KL}^{token}(\tau) = \sum_{t=1}^{|\tau|} \log \frac{\pi_\theta(a_t|s_t)}{\pi_{ref}(a_t|s_t)}
\end{equation*}
where the sum runs over all tokens in trajectory $\tau$. The RL objective includes this as a penalty:
\begin{equation*}
    J(\theta) = \mathbb{E}_{\tau \sim \pi_\theta}\left[R(\tau) - w \cdot D_{KL}^{token}(\tau)\right]
\end{equation*}

\paragraph{Decomposition Under the Two-Stage Framework.}
A trajectory $\tau$ of length $T$ (i.e., $T$ sampling-decision cycles) consists of:
\begin{itemize}
    \item $T$ sampling actions, where the $k$-th sampling action $(A_k, T_k)$ comprises $L_k$ tokens
    \item $T$ decision actions $a_k \in \{\text{RESAMPLE}, \text{STOP}\}$, each a single token
\end{itemize}

The total number of tokens is $|\tau| = \sum_{k=1}^{T} L_k + T$. The KL penalty decomposes as:
\begin{equation*}
\resizebox{1.0\hsize}{!}{$
\begin{aligned}
D_{KL}^{\text{token}}(\tau) &= \underbrace{\sum_{k=1}^{T} \sum_{j=1}^{L_k} \log \frac{\pi_{\text{sample},\theta}(\text{token}_{k,j} \mid \text{context})}{\pi_{\text{sample},\text{ref}}(\text{token}_{k,j} \mid \text{context})}}_{\text{Sampling component}} \\
&\quad + \underbrace{\sum_{k=1}^{T} \log \frac{\pi_{d,\theta}(a_k \mid s_k)}{\pi_{d,\text{ref}}(a_k \mid s_k)}}_{\text{Decision component}}
\end{aligned}
$}
\end{equation*}

\paragraph{Immediate Penalties.}
Define the immediate KL penalty attributed to each action:

\textbf{For sampling actions:}
\begin{equation*}
    d_k^{sample} = \sum_{j=1}^{L_k} \log \frac{\pi_{sample,\theta}(token_{k,j}|context)}{\pi_{sample,ref}(token_{k,j}|context)}
\end{equation*}
This is a sum over $L_k$ terms. If individual token-level log-ratios have typical magnitude $\delta$ (which may be positive or negative depending on whether the current policy assigns higher or lower probability than the reference), then:
\begin{equation*}
    |d_k^{sample}| \leq \sum_{j=1}^{L_k} \left|\log \frac{\pi_{sample,\theta}(token_{k,j}|\cdot)}{\pi_{sample,ref}(token_{k,j}|\cdot)}\right| \sim O(L_k \cdot \delta)
\end{equation*}

\textbf{For decision actions:}
\begin{equation*}
    d_k^{decision} = \log \frac{\pi_{d,\theta}(a_k|s_k)}{\pi_{d,ref}(a_k|s_k)}
\end{equation*}
This is a single scalar:
\begin{equation*}
    |d_k^{decision}| \sim O(\delta)
\end{equation*}

The ratio of magnitudes is approximately $L_k : 1$. For typical reasoning traces with $L_k \approx 100$--$500$ tokens, this represents a two-orders-of-magnitude asymmetry.

\paragraph{Q-Value Recursions.}
\begin{proof}
The policy gradient for the KL penalty takes the standard form:
\begin{equation*}
\resizebox{1.0\hsize}{!}{$
\nabla_\theta \mathbb{E}_{\pi_\theta}[D_{\text{KL}}^{\text{token}}(\tau)] = \mathbb{E}_{\tau \sim \pi_\theta}\left[\sum_t \nabla_\theta \log \pi_\theta(a_t \mid s_t) \cdot Q^{\text{KL}}(s_t, a_t)\right]
$}
\end{equation*}
where $Q^{KL}(s_t, a_t)$ is the expected cumulative KL penalty from taking action $a_t$ at state $s_t$.

We derive the Q-values by working backwards from the terminal state.

\textbf{Terminal state (STOP at step $T$):}
\begin{equation*}
    Q_d^{KL}(s_T, \text{STOP}) = d_T^{decision}
\end{equation*}
There are no future actions, so the Q-value equals the immediate penalty. Magnitude: $O(1)$.

\textbf{Final sampling action (step $T$):}
\begin{equation*}
    Q_{sample}^{KL}(s_{T-1}, a'_T) = d_T^{sample} + \gamma \cdot Q_d^{KL}(s_T, \text{STOP})
\end{equation*}
Substituting:
\begin{align*}
    Q_{sample}^{KL}(s_{T-1}, a'_T) & = d_T^{sample} + \gamma \cdot d_T^{decision} \\ & \sim O(L_T) + O(1) \approx O(L_T)
\end{align*}
The immediate sampling penalty dominates.

\textbf{Penultimate decision (RESAMPLE at step $T-1$):}
\begin{equation*}
    \resizebox{1.0\hsize}{!}{$
    Q_d^{KL}(s_{T-1}, \text{RESAMPLE}) = d_{T-1}^{decision} + \gamma \cdot \mathbb{E}_{\pi_\theta}\left[Q_{sample}^{KL}(s_{T-1}, a'_T)\right]
    $}
\end{equation*}
Magnitude:
\begin{equation*}
    \resizebox{1.0\hsize}{!}{$
    Q_d^{KL}(s_{T-1}, \text{RESAMPLE}) \sim O(1) + \gamma \cdot O(L_T) \approx O(\gamma L_T)
    $}
\end{equation*}
The future sampling penalty dominates, but note: the \emph{immediate} contribution is only $O(1)$.

\textbf{General recursion (for $k < T$):}
\begin{equation*}
\resizebox{1.0\hsize}{!}{$
\begin{aligned}
Q_{\text{sample}}^{\text{KL}}(s_{k-1}, a'_k) &= d_k^{\text{sample}} + \gamma \cdot \mathbb{E}_{\pi_\theta}\left[Q_d^{\text{KL}}(s_k, a''_k)\right] \\
Q_d^{\text{KL}}(s_k, \text{RESAMPLE}) &= d_k^{\text{decision}} + \gamma \cdot \mathbb{E}_{\pi_\theta}\left[Q_{\text{sample}}^{\text{KL}}(s_k, a'_{k+1})\right]
\end{aligned}
$}
\end{equation*}
\textbf{Scale separation analysis.}
Define the accumulated future KL from state $s_k$:
\begin{equation*}
    V^{KL}(s_k) = \mathbb{E}_{\pi_\theta}\left[\sum_{t \geq k} \gamma^{t-k} d_t \,\bigg|\, s_k\right]
\end{equation*}

For the sampling policy at step $k$:
\begin{equation*}
    Q_{sample}^{KL}(s_{k-1}, a'_k) = \underbrace{d_k^{sample}}_{\sim O(L_k)} + \gamma \cdot V^{KL}(s_k)
\end{equation*}

For the decision policy at step $k$:
\begin{equation*}
    Q_d^{KL}(s_k, a''_k) = \underbrace{d_k^{decision}}_{\sim O(1)} + \gamma \cdot V^{KL}(s_{k+1})
\end{equation*}

Although $V^{KL}(s_k)$ and $V^{KL}(s_{k+1})$ are similar in magnitude (both accumulate future penalties), the \emph{immediate} contributions differ by a factor of $L_k$.

\textbf{Impossibility of unified sufficient statistic.}
Suppose a single sufficient statistic $\Phi : \mathcal{S} \to \mathbb{R}$ exists such that:
\begin{align*}
    Q_{sample}^{KL}(s_{k-1}, a'_k) &= f_{sample}(s_{k-1}, a'_k, \Phi(s_k)) \\
    Q_d^{KL}(s_k, a''_k) &= f_d(s_k, a''_k, \Phi(s_{k+1}))
\end{align*}

From the Bellman-style recursions, $\Phi$ must satisfy:
\begin{align*}
    \Phi(s_k) &= O(L_k) + \gamma \Phi(s_{k+1}) \quad \text{(from sampling Q-value)} \\
    \Phi(s_k) &= O(1) + \gamma \Phi(s_{k+1}) \quad \text{(from decision Q-value)}
\end{align*}

Subtracting: $0 = O(L_k) - O(1)$, which is a contradiction when $L_k \gg 1$.

Therefore, no single sufficient statistic $\Phi$ exists. The Q-values require distinct information structures:
\begin{align*}
    \Phi_{sample}(s_k) &= d_k^{sample} + \gamma \Phi_d(s_{k+1}) \sim O(L_k) + \gamma(\cdot) \\
    \Phi_d(s_k) &= d_k^{decision} + \gamma \Phi_{sample}(s_k) \sim O(1) + \gamma(\cdot)
\end{align*}

This establishes Unbalanced Gradient Attribution.
\end{proof}

\newpage

\begin{figure*}[h!]
\centering

\resizebox{0.8\textwidth}{!}{%
\begin{tikzpicture}[node distance=0.3cm, auto]

\node[font=\large\bfseries] at (5.5,6.3) {(a) Trajectory Generation Process};

\node[query] (q) at (0,4.5) {$Q$};

\node[sample, right=of q] (w) at (1.5,4.5) {Sample $\hat{W}$};
\node[font=\tiny, below=0.02cm of w, text=blue!70] {$\pi_{\text{sample}}$ (lots of tokens)};

\node[decision, right=of w] (r1) at (5.2,4.5) {R};
\node[font=\tiny, below=0.02cm of r1, text=orange!70] {$\pi_d$ (few tokens)};

\node[sample, right=of r1] (c) at (6.8,4.5) {Sample $\hat{C}$};
\node[font=\tiny, below=0.02cm of c, text=blue!70] {$\pi_{\text{sample}}$ (lots of tokens)};

\node[decision, right=of c] (s) at (10.5,4.5) {S};
\node[font=\tiny, below=0.02cm of s, text=orange!70] {$\pi_d$ (few tokens)};

\draw[arrow] (q) -- (w);
\draw[arrow] (w) -- (r1);
\draw[arrow] (r1) -- (c);
\draw[arrow] (c) -- (s);

\draw[arrow, dashed, thick, gray!60, bend right=60] (r1.north) to node[midway, above, font=\tiny, text=gray!60] {Reject and Resample} (w.north);

\draw[->] (-0.5,3.5) -- (11.5,3.5) node[right, font=\small] {\#Try};
\node[font=\tiny] at (0,3.3) {$t=0$};
\node[font=\tiny] at (3,3.3) {$t=k$};
\node[font=\tiny] at (6,3.3) {$t=k$};
\node[font=\tiny] at (8.5,3.3) {$t=N$};

\node[font=\large\bfseries] at (5.5,2.5) {Reward Calculation};

\node[draw, fill=yellow!10, rounded corners, font=\small, align=center] at (1.0,0.6) {
    Trajectory $\tau = (W, \text{RESAMPLE}, C, \text{STOP})$ \\[0.08cm]
    $\downarrow$ \\[0.03cm]
    Final Answer: $C = A^*$ \checkmark \\[0.08cm]
    $\downarrow$ \\[0.03cm]
    $R(\tau) = +1$ \\[0.08cm]
    $\downarrow$ \\[0.03cm]
    Advantage: $A_i = +$
};

\node[draw, fill=yellow!10, rounded corners, font=\small, align=center] at (8.0,0.6) {
    Trajectory $\tau = (W, \text{RESAMPLE}, C, \text{STOP})$ \\[0.08cm]
    $\downarrow$ \\[0.03cm]
    Final Answer: $C \neq A^*$ \checkmark \\[0.08cm]
    $\downarrow$ \\[0.03cm]
    $R(\tau) = +0$ \\[0.08cm]
    $\downarrow$ \\[0.03cm]
    Advantage: $A_i = -$
};

\node[font=\large\bfseries] at (5.5,-1.5) {Gradient Propagation};

\coordinate (pos_w) at (3,-2.5);
\coordinate (pos_r1) at (5.2,-2.5);
\coordinate (pos_c) at (8.5,-2.5);
\coordinate (pos_s) at (10.5,-2.5);

\node[font=\small, align=right] at (-0.8,-2.5) {Reward\\Track:};
\draw[gradient_arrow, line width=2.5pt, blue!70] (pos_s) -- (pos_c);
\draw[gradient_arrow, line width=2.5pt, blue!70] (pos_c) -- (pos_r1);
\draw[gradient_arrow, line width=2.5pt, blue!70] (pos_r1) -- (pos_w);
\draw[gradient_arrow, line width=2.5pt, blue!70] (pos_w) -- (0.5,-2.5);

\node[font=\scriptsize, blue!70, above] at (9.5,-2.5) {$A_i$};
\node[font=\scriptsize, blue!70, above] at (6.8,-2.5) {$A_i$};
\node[font=\scriptsize, blue!70, above] at (4.1,-2.5) {$A_i$};
\node[font=\scriptsize, blue!70, above] at (1.8,-2.5) {$A_i$};

\node[font=\small, align=right] at (-0.8,-3.5) {KL\\Track:};
\coordinate (kl_w) at (3,-3.5);
\coordinate (kl_r1) at (5.2,-3.5);
\coordinate (kl_c) at (8.5,-3.5);
\coordinate (kl_s) at (10.5,-3.5);

\draw[gradient_arrow, line width=0.8pt, red!70] (kl_s) -- (kl_c);
\draw[gradient_arrow, line width=3.5pt, red!70] (kl_c) -- (kl_r1);
\draw[gradient_arrow, line width=0.8pt, red!70] (kl_r1) -- (kl_w);
\draw[gradient_arrow, line width=3.5pt, red!70] (kl_w) -- (0.5,-3.5);

\node[font=\scriptsize, red!70, above] at (9.5,-3.5) {short length Agg $Q$};
\node[font=\scriptsize, red!70, above] at (6.8,-3.5) {long length Agg $Q$};
\node[font=\scriptsize, red!70, above] at (4.1,-3.5) {short length Agg $Q$};
\node[font=\scriptsize, red!70, above] at (1.8,-3.5) {long length Agg $Q$};
\end{tikzpicture}
}

\vspace{0.5cm} 

\resizebox{\textwidth}{!}{%
\begin{tikzpicture}[node distance=1cm, auto, scale=1, every node/.style={scale=1}]

\node[font=\large\bfseries] at (0,5.2) {(b) Surrogate Reward (Balanced)};
\node[font=\large\bfseries] at (7,5.2) {(c) KL Penalty (Unbalanced)};

\node[reward] (reward_left) at (0,4) {Reward Gradient\\$\nabla_\theta J_{\text{reward}}$};
\node[block, fill=green!10, text width=6.5em] (decomp_left) at (0,2.3) {$\sum [\nabla_\theta \log \pi \cdot Q^\pi]$};
\draw[arrow] (reward_left) -- (decomp_left) node[midway,right,font=\tiny] {$A_i$};

\node[draw, circle, fill=green!30, inner sep=1pt, font=\tiny] (q_shared) at (0,1.1) {$Q^\pi = A_{ij}$};
\node[policy] (sample_left) at (-1.8,0) {$\nabla_\theta \pi_{\text{sample}}$};
\node[policy] (decision_left) at (1.8,0) {$\nabla_\theta \pi_d$};

\draw[arrow, green!60!black, very thick] (decomp_left) -- (q_shared);
\draw[arrow, green!60!black, very thick] (q_shared) -- (sample_left) node[midway,left,font=\tiny] {$\times A_i$};
\draw[arrow, green!60!black, very thick] (q_shared) -- (decision_left) node[midway,right,font=\tiny] {$\times A_i$};

\node[draw, fill=green!5, text width=6.2cm, font=\tiny, rounded corners, inner sep=3pt] at (0,-1.8) {
    \begin{align*}
    Q^{\pi}_{\text{sample}} &\approx Q^{\pi}_d \approx A_{ij} \text{ (length discounted $A_i$)}\\
    \nabla_\theta J &\approx \sum_k \nabla_\theta \log \pi_{\text{sample}}(A_k|s_{k-1}) \cdot A_{ik} \\
    &\quad + \sum_k \nabla_\theta \log \pi_d(a_k|s_k) \cdot A_{ik}
    \end{align*}
};

\node[font=\scriptsize, text width=5cm, align=center] at (0,-3.2) {
    \textcolor{green!60!black}{\textbf{balanced gradient attribution}}
};

\node[reward, fill=red!20] (reward_right) at (7,4) {KL Gradient\\$\nabla_\theta J_{\text{KL}}$};
\node[block, fill=red!10, text width=6.5em] (decomp_right) at (7,2.3) {$\sum [\nabla_\theta \log \pi \cdot Q^{\pi,\text{KL}}]$};
\draw[arrow] (reward_right) -- (decomp_right) node[midway,right,font=\tiny] {$d_k$};

\node[draw, circle, fill=yellow!30, inner sep=1pt, font=\tiny] (q_sample) at (5.5,1.2) {$Q^{\text{KL}}_{\text{sample}}$};
\node[draw, circle, fill=yellow!30, inner sep=1pt, font=\tiny] (q_decision) at (8.5,1.2) {$Q^{\text{KL}}_d$};
\node[font=\tiny, color=red!60!black] at (7,0.7) {(Different $Q$)};

\node[policy, fill=red!10] (sample_right) at (5.5,-0.1) {$\nabla_\theta \pi_{\text{sample}}$};
\node[policy, fill=red!10] (decision_right) at (8.5,-0.1) {$\nabla_\theta \pi_d$};

\draw[arrow, red!60!black] (decomp_right) -- (q_sample);
\draw[arrow, red!60!black] (decomp_right) -- (q_decision);
\draw[arrow] (q_sample) -- (sample_right) node[midway,left,font=\tiny] {$\times Q^{\text{KL}}_{\text{sample}}$};
\draw[arrow] (q_decision) -- (decision_right) node[midway,right,font=\tiny] {$\times Q^{\text{KL}}_d$};

\draw[bidarrow, bend left=35, line width=1pt] (q_sample) to node[midway,above,font=\tiny] {} (q_decision);
\draw[bidarrow, bend right=35, line width=1pt] (q_sample) to node[midway,below,font=\tiny] {circular dependency} (q_decision);

\node[draw, fill=red!5, text width=6.2cm, font=\tiny, rounded corners, inner sep=3pt] at (7,-1.8) {
    \begin{align*}
    Q^{\text{KL}}_{\text{sample}} &= d^{\text{sample}} + \mathbb{E}_{\pi_d}[\gamma \cdot Q^{\text{KL}}_d] \\
    Q^{\text{KL}}_d &= d^{\text{decision}} + \mathbb{E}_{\pi_{\text{sample}}}[\gamma \cdot Q^{\text{KL}}_{\text{sample}}]
    \end{align*}
    \textcolor{red!60!black}{\textbf{$Q^{\text{KL}}_{\text{sample}} \not\approx Q^{\text{KL}}_d$ with circular dependency}}
};

\node[font=\scriptsize, text width=5cm, align=center] at (7,-3.2) {
    \textcolor{red!60!black}{\textbf{Imbalanced gradient attribution}}
};

\end{tikzpicture}
}

\end{figure*}

\clearpage

\section{Alternative Proof of Theorem 3.1 and Theorem 3.2}\label{app:alt_proof}

The probability of a trajectory $\tau$ under the combined policy can be expressed as:
\begin{equation}
\resizebox{1.0\hsize}{!}{$
\begin{aligned}
P(\tau \mid Q;\theta) &= \left(\prod_{k=1}^{T} \pi_{\text{sample}}(A_k, T_k \mid s_{k-1};\theta)\right) \\
&\quad \cdot \left(\prod_{k=1}^{T-1} \pi_d(\text{RESAMPLE} \mid s_k;\theta)\right) \cdot \pi_d(\text{STOP} \mid s_T;\theta)
\end{aligned}
$}
\end{equation}

Taking the logarithm:
\begin{equation}
\resizebox{1.0\hsize}{!}{$
\begin{aligned}
\log P(\tau \mid Q;\theta) &= \sum_{k=1}^{T} \log \pi_{\text{sample}}(A_k, T_k \mid s_{k-1};\theta) \\
&\quad + \sum_{k=1}^{T-1} \log \pi_d(\text{RESAMPLE} \mid s_k;\theta) + \log \pi_d(\text{STOP} \mid s_T;\theta)
\end{aligned}
$}
\end{equation}

\textbf{Important:} The sampling log-probability $\log \pi_{sample}(A_k, T_k|s_{k-1})$ decomposes at the token level as:
\begin{equation}
\resizebox{1.0\hsize}{!}{$
\begin{aligned}
\log \pi_{\text{sample}}(A_k, T_k \mid s_{k-1}) &= \sum_{j=1}^{L_k} \log \pi_{\text{sample}}(\text{token}_{k,j} \mid s_{k-1}, \\
&\quad \text{token}_{k,1}, \ldots, \text{token}_{k,j-1})
\end{aligned}
$}
\end{equation}

The gradient with respect to $\theta$:
\begin{equation}
\resizebox{1.0\hsize}{!}{$
\begin{aligned}
\nabla_\theta \log P(\tau \mid Q;\theta) &= \sum_{k=1}^{T} \nabla_\theta \log \pi_{\text{sample}}(A_k, T_k \mid s_{k-1};\theta) \\
&\quad + \sum_{k=1}^{T-1} \nabla_\theta \log \pi_d(\text{RESAMPLE} \mid s_k;\theta) + \nabla_\theta \log \pi_d(\text{STOP} \mid s_T;\theta)
\end{aligned}
$}
\end{equation}

For convenience, we denote $A_0, T_0 = \emptyset$.

\subsection{Policy Gradient Decomposition}

Using the policy gradient theorem adapted to our factorized policy structure (see Appendix~\ref{app:proof_pg}):
\begin{equation}
\resizebox{1.0\hsize}{!}{$
\begin{aligned}
\nabla_\theta J(\theta) &= \mathbb{E}_{\tau \sim \pi_\theta}\Bigg[\sum_{t=0}^{T} \Bigg(\nabla_\theta \log \pi_{\text{sample},\theta}(a'_t \mid s_t) \cdot Q^\pi(s_t, a'_t) \\
&\quad + \nabla_\theta \log \pi_{d,\theta}(a''_t \mid s_t) \cdot Q^\pi(s_t, a''_t)\Bigg)\Bigg]
\end{aligned}
$}
\end{equation}

\subsection{Attribution Analysis for Surrogate Reward}

\paragraph{Surrogate Reward Term.}
\begin{equation}
    L_{clip}(\theta) = \mathbb{E}_{\tau_i \sim \pi_{old}}\left[\frac{\pi_\theta(\tau_i|Q_i)}{\pi_{old}(\tau_i|Q_i)} A_i\right]
\end{equation}

where $A_i$ is the group-relative advantage computed via GRAE:
\begin{equation}
    A_i = \frac{R(\tau_i) - \text{mean}(R(\tau_1), \ldots, R(\tau_G))}{\text{std}(R(\tau_1), \ldots, R(\tau_G))}
\end{equation}

The probability ratio factorizes as:
\begin{equation}
\resizebox{1.0\hsize}{!}{$
\begin{aligned}
\frac{\pi_\theta(\tau_i \mid Q_i)}{\pi_{\text{old}}(\tau_i \mid Q_i)} &= \underbrace{\left(\prod_{k=1}^{T_i} \frac{\pi_{\theta,\text{sample}}(A_k, T_k \mid s_{k-1})}{\pi_{\text{old},\text{sample}}(A_k, T_k \mid s_{k-1})}\right)}_{\text{Sampling ratio}} \\
&\quad \cdot \underbrace{\left(\prod_{k=1}^{T_i-1} \frac{\pi_{\theta,d}(\text{RESAMPLE} \mid s_k)}{\pi_{\text{old},d}(\text{RESAMPLE} \mid s_k)}\right)}_{\text{Resampling ratio}} \\
&\quad \cdot \underbrace{\frac{\pi_{\theta,d}(\text{STOP} \mid s_{T_i})}{\pi_{\text{old},d}(\text{STOP} \mid s_{T_i})}}_{\text{Stopping ratio}}
\end{aligned}
$}
\end{equation}

Neglecting the clipping for analytical tractability, the gradient is:
\begin{equation}
\resizebox{1.0\hsize}{!}{$
\begin{aligned}
\nabla_\theta L_{\text{reward}} &\propto A_i \Bigg[\sum_{k=1}^{T_i} \nabla_\theta \log \pi_{\theta,\text{sample}}(a'_k \mid s_{k-1}) \\
&\quad + \sum_{k=0}^{T_i} \nabla_\theta \log \pi_{\theta,d}(a''_k \mid s_k)\Bigg]
\end{aligned}
$}
\end{equation}

\textbf{Balanced attribution:} The advantage $A_i$ multiplies both sampling and decision gradients equally. Both Q-values reduce to the same sufficient statistic:
\begin{equation}
    Q_d^{reward} = Q_{sample}^{reward} \sim \gamma^{\sum_{j=0}^{T_i-k} len(A_{k+j}, T_{k+j})} A_i
\end{equation}

\subsection{Attribution Analysis for KL Penalty}

\paragraph{Token-Level KL Decomposition.}
The KL penalty decomposes at the token level:
\begin{equation}
\resizebox{1.0\hsize}{!}{$
\begin{aligned}
D_{KL}^{\text{token}}(\tau_i) &= \sum_{k=1}^{T_i} \underbrace{\sum_{j=1}^{L_k} \log \frac{\pi_{\text{sample},\theta}(\text{token}_{k,j} \mid \cdot)}{\pi_{\text{sample},\text{ref}}(\text{token}_{k,j} \mid \cdot)}}_{d_k^{\text{sample}}} \\
&\quad + \sum_{k=1}^{T_i} \underbrace{\log \frac{\pi_{d,\theta}(a_k \mid s_k)}{\pi_{d,\text{ref}}(a_k \mid s_k)}}_{d_k^{\text{decision}}}
\end{aligned}
$}
\end{equation}

The immediate penalties are:
\begin{itemize}
    \item \textbf{Sampling:} $d_k^{sample} = \sum_{j=1}^{L_k} \log(\pi_\theta/\pi_{ref})$ is a sum of $L_k$ terms $\Rightarrow O(L_k)$
    \item \textbf{Decision:} $d_k^{decision} = \log(\pi_\theta/\pi_{ref})$ is a single term $\Rightarrow O(1)$
\end{itemize}

\paragraph{Q-Value Recursions.}
Working backwards from the terminal state:

\textbf{Value of final STOP action:}
\begin{equation}
    Q_d^{KL}(s_{T_i}, \text{STOP}) = d_{T_i}^{decision} \sim O(1)
\end{equation}

\textbf{Value of preceding sampling action:}
\begin{equation}
\resizebox{1.0\hsize}{!}{$
\begin{aligned}
Q_{\text{sample}}^{\text{KL}}(s_{T_i-1}, a'_{T_i}) &= d_{T_i}^{\text{sample}} + \gamma \cdot Q_d^{\text{KL}}(s_{T_i}, \text{STOP}) \\
&\sim O(L_{T_i}) + O(1) \approx O(L_{T_i})
\end{aligned}
$}
\end{equation}

\textbf{Value of RESAMPLE action:}
\begin{equation}
\resizebox{1.0\hsize}{!}{$
\begin{aligned}
Q_d^{\text{KL}}(s_k, \text{RESAMPLE}) &= d_k^{\text{decision}} + \gamma \cdot \mathbb{E}_{\pi_\theta}[Q_{\text{sample}}^{\text{KL}}(s_k, a'_{k+1})] \\
&\sim O(1) + O(\gamma L_{k+1})
\end{aligned}
$}
\end{equation}

\paragraph{Asymmetric Gradient Attribution.}
The gradient contributions are:
\begin{equation}
\resizebox{1.0\hsize}{!}{$
\begin{aligned}
\text{Sampling:} \quad & \sum_{k=1}^{T_i} \nabla_\theta \log \pi_{\text{sample},\theta}(a'_k \mid s_{k-1}) \cdot Q_{\text{sample}}^{\text{KL}}(s_{k-1}, a'_k) \\
\text{Decision:} \quad & \sum_{k=1}^{T_i} \nabla_\theta \log \pi_{d,\theta}(a''_k \mid s_k) \cdot Q_d^{\text{KL}}(s_k, a''_k)
\end{aligned}
$}
\end{equation}

The Q-values have incompatible sufficient statistics:
\begin{itemize}
    \item $Q_{sample}^{KL}$: dominated by immediate $O(L_k)$ penalty
    \item $Q_d^{KL}$: immediate penalty is $O(1)$, accumulates future sampling penalties
\end{itemize}

This structural asymmetry arising from the sum over $L_k$ tokens in sampling versus single-token decisions creates \textbf{Unbalanced Gradient Attribution}.

\newpage
\section{Numerical Illustration of Gradient Attribution}\label{app:numerical}

To make the theoretical derivations in the preceding sections concrete, this section presents a simplified numerical example focused on a key scenario: where correct and incorrect answers have the same length. The objective is to trace the flow of gradients from the two primary components of the Group Relative Policy Optimization (GRPO) objective---the surrogate reward\footnote{We also simplify the clipped part away} and the KL penalty---back to the parameters of the sampling policy ($\pi_{sample}$) and the decision policy ($\pi_{d}$).

This exercise will quantitatively demonstrate that the asymmetric regularization effect is not dependent on incorrect answers being more verbose. Instead, it is an architectural consequence of the length-weighted formulation of the KL penalty, which applies a significant penalty to any long generated sequence from $\pi_{sample}$ while applying a negligible penalty to the single-step actions of $\pi_{d}$.

\subsection{Scenario and Policy Parameterization}

We establish a minimal yet illustrative scenario to examine the gradient dynamics under the condition of equal answer lengths.

\paragraph{Problem Setup} The model is tasked with a simple arithmetic problem: "What is 7 * 8?".

\paragraph{Action Space}
\begin{itemize}
    \item \textbf{Sampling Policy ($\pi_{sample}$):} Generates an answer. We consider two outcomes: a Correct answer (C), "56", or a Wrong answer (W), "54".
    \item \textbf{Decision Policy ($\pi_{d}$):} Evaluates the generated answer and chooses to STOP or RESAMPLE.
\end{itemize}

\paragraph{Policy Parameterization} Each policy choice is modeled using a logistic sigmoid function, $\sigma(x) = 1 / (1 + e^{-x})$, applied to a single learnable logit ($\theta$).
\begin{itemize}
    \item \textbf{Sampling Policy ($\pi_{sample}$):} Governed by a single parameter, $\theta_s$. The probability of generating a correct answer is $P(C|\theta_s) = \sigma(\theta_s)$. We initialize $\theta_s = 0.4$, yielding $P(C) \approx 0.5987$ and $P(W) = 1 - P(C) \approx 0.4013$.
    \item \textbf{Decision Policy ($\pi_{d}$):} Conditioned on the answer from $\pi_{sample}$.
    \begin{itemize}
        \item For a correct answer, the stop probability is $P(\text{STOP}|C, \theta_{d,C}) = \sigma(\theta_{d,C})$. We initialize $\theta_{d,C} = 2.2$, yielding $P(\text{STOP}|C) \approx 0.9002$.
        \item For a wrong answer, the resample probability is $P(\text{RESAMPLE}|W, \theta_{d,W}) = \sigma(\theta_{d,W})$. We initialize $\theta_{d,W} = 1.4$, yielding $P(\text{RESAMPLE}|W) \approx 0.8022$.
    \end{itemize}
\end{itemize}

\paragraph{Length Assignment} This is the critical modification for this example. We set the lengths of both correct and incorrect answers to be equal and significant.
\begin{itemize}
    \item Length of Correct answer: $len(C) = 8$ tokens.
    \item Length of Wrong answer: $len(W) = 8$ tokens.
\end{itemize}

\paragraph{Reference Policy ($\pi_{orig}$)} The KL penalty regularizes the current policy ($\pi_{\theta}$) against a reference policy ($\pi_{orig}$), typically the initial SFT model.
\begin{itemize}
    \item $\theta_{s,orig} = 0.3 \implies P_{orig}(C) \approx 0.5744, P_{orig}(W) \approx 0.4256$.
    \item $\theta_{d,C,orig} = 2.0 \implies P_{orig}(\text{STOP}|C) \approx 0.8808$.
    \item $\theta_{d,W,orig} = 1.2 \implies P_{orig}(\text{RESAMPLE}|W) \approx 0.7685$.
\end{itemize}

\subsection*{Part 1: Symmetric Gradient Push from the Surrogate Reward}

First, we analyze the gradient from the surrogate reward. This calculation is entirely independent of sequence length, demonstrating its symmetric nature.

\paragraph{Trajectory and Advantage} We consider a single successful trajectory $\tau_i$: the model initially generates a Wrong answer, correctly chooses to RESAMPLE, then generates a Correct answer and correctly chooses to STOP.
\
GRPO is a "critic-less" algorithm that estimates advantage relative to a group of trajectories. We assume this trajectory is better than the group average and assign it a normalized advantage of $A_i = +0.5$.

\paragraph{Gradient Calculation} The gradient of the reward objective is $\nabla_{\theta} J_{Reward}(\theta) = A_i \cdot \nabla_{\theta} \log P(\tau_i|\theta)$. The log-probability of the trajectory is:
\
The gradient for each parameter is calculated as follows:
\begin{itemize}
    \item \textbf{For $\theta_s$ (Sampling Policy):}
    $\begin{aligned}[t]
        \nabla_{\theta_s} J_{\text{Reward}} &= A_i \cdot \\
        &= 0.5 \cdot [-\sigma(\theta_s) + (1 - \sigma(\theta_s))] \\
        &= 0.5 \cdot [-0.5987 + 0.4013] = \mathbf{-0.0987}
    \end{aligned}$

    \item \textbf{For $\theta_{d,W}$ (Decision Policy):}
    $\begin{aligned}[t]
        \nabla_{\theta_{d,W}} J_{\text{Reward}} &= A_i \cdot \nabla_{\theta_{d,W}}\log P(\text{RESAMPLE}|W, \theta_{d,W}) \\
        &= 0.5 \cdot (1 - \sigma(\theta_{d,W})) \\
        &= 0.5 \cdot (1 - 0.8022) = \mathbf{+0.0989}
    \end{aligned}$

    \item \textbf{For $\theta_{d,C}$ (Decision Policy):}
    $\begin{aligned}[t]
        \nabla_{\theta_{d,C}} J_{\text{Reward}} &= A_i \cdot \nabla_{\theta_{d,C}}\log P(\text{STOP}|C, \theta_{d,C}) \\
        &= 0.5 \cdot (1 - \sigma(\theta_{d,C})) \\
        &= 0.5 \cdot (1 - 0.9002) = \mathbf{+0.0499}
    \end{aligned}$
\end{itemize}

The advantage $A_i$ is applied as a scalar multiplier to all responsible parameters. The resulting gradients are of a similar order of magnitude, demonstrating a \textit{symmetric push} for improvement.

\subsection*{Part 2: Asymmetric Gradient Drag from the KL Penalty}

Next, we analyze the gradient from the KL penalty. The asymmetry arises because the immediate penalty for a sampling action is explicitly weighted by its length, whereas the penalty for a decision action is not.

\paragraph{Calculating KL Penalties and Q-Values} We first calculate the immediate KL penalty ($d_k$) for each step in our trajectory. This measures the "information loss" when the current policy deviates from the reference policy.
\begin{itemize}
    \item \textbf{Step 1 - Sample W:} $d_1^{sample} = 8 \cdot \left(-\log\frac{0.4013}{0.4256}\right) = 8 \cdot (0.0588) \approx +0.4704$
    \item \textbf{Step 1 - Resample:} $d_1^{decision} = -\log\frac{0.8022}{0.7685} \approx -0.0429$
    \item \textbf{Step 2 - Sample C:} $d_2^{sample} = 8 \cdot \left(-\log\frac{0.5987}{0.5744}\right) = 8 \cdot (-0.0415) \approx -0.3320$
    \item \textbf{Step 2 - Stop:} $d_2^{decision} = -\log\frac{0.9002}{0.8808} \approx -0.0218$
\end{itemize}
Next, we compute the state-action values ($Q^{\pi,KL}$) by working backward from the end of the trajectory (with discount factor $\gamma=1$).
\begin{itemize}
    \item \textbf{Value of final action (STOP):} $Q_{KL}(s_2, \text{STOP}) = d_2^{decision} = -0.0218$
    \item \textbf{Value of preceding action (Sample C):} $Q_{KL}(s_1, \text{sample } C) = d_2^{sample} + Q_{KL}(s_2, \text{STOP}) = -0.3320 - 0.0218 = -0.3538$
    \item \textbf{Value of preceding action (RESAMPLE):} $Q_{KL}(s_1, \text{RESAMPLE}) = d_1^{decision} + Q_{KL}(s_1, \text{sample } C) = -0.0429 - 0.3538 = -0.3967$
    \item \textbf{Value of first action (Sample W):} $Q_{KL}(s_0, \text{sample } W) = d_1^{sample} + Q_{KL}(s_1, \text{RESAMPLE}) = 0.4704 - 0.3967 = +0.0737$
\end{itemize}

\paragraph{Gradient Calculation} The gradient of the KL penalty objective is proportional to $\sum_{t} \nabla_{\theta} \log \pi_t \cdot Q_{KL,t}$. This gradient acts as a "drag" force, pulling the policy back toward the reference.

\begin{itemize}
    \item \textbf{For $\theta_s$ (Sampling Policy):}
    
    \noindent
    \resizebox{0.9\hsize}{!}{$
    \begin{aligned}
        \nabla_{\theta_s} J_{\text{KL}} &\propto \nabla_{\theta_s}\log P(W|\theta_s) \cdot Q_{\text{KL}}(s_0, W) \\
        &\quad + \nabla_{\theta_s}\log P(C|\theta_s) \cdot Q_{\text{KL}}(s_1, C) \\
        &\propto (-\sigma(\theta_s)) \cdot (0.0737) + (1-\sigma(\theta_s)) \cdot (-0.3538) \\
        &\propto (-0.5987) \cdot (0.0737) + (0.4013) \cdot (-0.3538) \\
        &= -0.0441 - 0.1420 = \mathbf{-0.1861}
    \end{aligned}
    $}

    \item \textbf{For $\theta_{d,W}$ (Decision Policy):}
    
    \noindent
    \resizebox{0.9\hsize}{!}{$
    \begin{aligned}
        \nabla_{\theta_{d,W}} J_{\text{KL}} &\propto \nabla_{\theta_{d,W}}\log P(\text{RESAMPLE}|W) \cdot Q_{\text{KL}}(s_1, \text{RESAMPLE}) \\
        &\propto (1 - \sigma(\theta_{d,W})) \cdot (-0.3967) \\
        &= (0.1978) \cdot (-0.3967) = \mathbf{-0.0785}
    \end{aligned}
    $}

    \item \textbf{For $\theta_{d,C}$ (Decision Policy):}
    
    \noindent
    \resizebox{0.9\hsize}{!}{$
    \begin{aligned}
        \nabla_{\theta_{d,C}} J_{\text{KL}} &\propto \nabla_{\theta_{d,C}}\log P(\text{STOP}|C) \cdot Q_{\text{KL}}(s_2, \text{STOP}) \\
        &\propto (1 - \sigma(\theta_{d,C})) \cdot (-0.0218) \\
        &= (0.0998) \cdot (-0.0218) = \mathbf{-0.0022}
    \end{aligned}
    $}
\end{itemize}

The results are starkly different from the reward gradients. The gradient magnitude for the sampling policy, $|\nabla_{\theta_s} J_{KL}| \approx 0.186$, is more than double that for the first decision parameter ($|\nabla_{\theta_{d,W}} J_{KL}| \approx 0.079$) and over 80 times larger than for the second ($|\nabla_{\theta_{d,C}} J_{KL}| \approx 0.002$). This is the \textit{asymmetric drag}.

\begin{table*}[h!]
\centering
\caption{Breakdown of KL Penalty Gradient Attribution}
\label{tab:kl_breakdown}
\resizebox{\textwidth}{!}{%
\begin{tabular}{@{}llcccccc@{}}
\toprule
Step & Action & $d_k$ (Penalty) & $V_{KL}(s')$ (Future) & $Q_{KL}$ (Total) & $\nabla_{\theta}\log\pi$ & Parameter & Contribution \\
\midrule
2 & STOP & -0.0218 & 0 & \textbf{-0.0218} & $1 - \sigma(\theta_{d,C}) = 0.0998$ & $\theta_{d,C}$ & -0.0022 \\
2 & Sample C & -0.3320 & -0.0218 & \textbf{-0.3538} & $1 - \sigma(\theta_s) = 0.4013$ & $\theta_s$ & -0.1420 \\
1 & RESAMPLE & -0.0429 & -0.3538 & \textbf{-0.3967} & $1 - \sigma(\theta_{d,W}) = 0.1978$ & $\theta_{d,W}$ & -0.0785 \\
1 & Sample W & +0.4704 & -0.3967 & \textbf{+0.0737} & $-\sigma(\theta_s) = -0.5987$ & $\theta_s$ & -0.0441 \\
\bottomrule
\end{tabular}%
}
\end{table*}

\subsection{Synthesis: The Net Parameter Update}

The final update to the policy parameters is the sum of the gradients from the surrogate reward (the "push") and the KL penalty (the "drag").

\begin{table*}[h!]
\centering
\caption{Summary of Gradient Attribution and Net Update}
\label{tab:net_update}
\begin{tabular}{@{}lccccc@{}}
\toprule
Parameter & Policy & Reward (Push) & KL (Drag) & Net Gradient & Interpretation \\
\midrule
$\theta_s$ & $\pi_{sample}$ & -0.0987 & -0.1861 & \textbf{-0.2848} & Heavily regularized; stable. \\
$\theta_{d,W}$ & $\pi_{d}$ & +0.0989 & -0.0785 & \textbf{+0.0204} & Learning signal overcomes drag. \\
$\theta_{d,C}$ & $\pi_{d}$ & +0.0499 & -0.0022 & \textbf{+0.0477} & Learning driven by reward. \\
\bottomrule
\end{tabular}
\end{table*}

\section{Gradient Attribution Properties of SFT and DFT}\label{app:sft_dft}

We analyze the gradient attribution properties of Supervised Fine-Tuning (SFT) and Dynamic Fine-Tuning (DFT) within our two-stage decision-sampling framework. We show that SFT exhibits imbalanced gradient attribution due to policy-entangled Q-values, while DFT removes this entanglement, achieving improved (though not perfectly balanced) gradient attribution.

\subsection{SFT as Policy Gradient with Implicit Reward}
\begin{lemma}
The SFT gradient is equivalent to a policy gradient with implicit per-token reward:
\begin{equation}
\resizebox{1.0\hsize}{!}{$
\begin{aligned}
\nabla_\theta \mathcal{L}_{\text{SFT}} &= \mathbb{E}_{\tau \sim \mathcal{D}}\Bigg[\sum_{k=1}^{T} \sum_{j=1}^{L_k} \nabla_\theta \log \pi_{\text{sample}}(\text{token}_{k,j} \mid \cdot) \\
&\quad + \sum_{k=1}^{T} \nabla_\theta \log \pi_d(a_k \mid s_k)\Bigg]
\end{aligned}
$}
\end{equation}
where the implicit reward at each token is $r_t = 1$, but weighted by $1/\pi_\theta$ in the gradient (since $\nabla_\theta \log \pi = \frac{1}{\pi}\nabla_\theta \pi$).
\end{lemma}

\subsection{SFT: Imbalanced gradient attribution via Policy-Entanglement}
\begin{theorem}[SFT has Imbalanced gradient attribution]
For SFT, the implicit Q-values are:
\begin{equation}
\resizebox{1.0\hsize}{!}{$
\begin{aligned}
Q^{\text{SFT}}_{\text{sample}}(s_{k-1}, \text{token}_{k,j}) &= \underbrace{\sum_{j'=j+1}^{L_k} \frac{1}{\pi_{\text{sample}}(\text{token}_{k,j'})}}_{\text{remaining tokens in step } k} + \underbrace{\frac{1}{\pi_d(a_k \mid s_k)}}_{\text{decision at step } k} \\
&\quad + \underbrace{\sum_{k'=k+1}^{T}\left(\sum_{j'=1}^{L_{k'}} \frac{1}{\pi_{\text{sample}}(\text{token}_{k',j'})} + \frac{1}{\pi_d(a_{k'} \mid s_{k'})}\right)}_{\text{future steps}} \\
Q^{\text{SFT}}_d(s_k, a_k) &= \sum_{k'=k+1}^{T}\left(\sum_{j'=1}^{L_{k'}} \frac{1}{\pi_{\text{sample}}(\text{token}_{k',j'})} + \frac{1}{\pi_d(a_{k'} \mid s_{k'})}\right)
\end{aligned}
$}
\end{equation}
These Q-values exhibit \textbf{policy-entanglement}:
\begin{itemize}
    \item $Q^{\text{SFT}}_{\text{sample}}$ depends on $\pi_d$ through the $1/\pi_d(a_k | s_k)$ and $1/\pi_d(a_{k'} | s_{k'})$ terms.
    \item $Q^{\text{SFT}}_d$ depends on $\pi_{\text{sample}}$ through the $1/\pi_{\text{sample}}(\text{token}_{k',j'})$ terms.
\end{itemize}
\end{theorem}

\begin{proof}
The SFT gradient can be written in policy gradient form:
\begin{equation}
    \nabla_\theta \mathcal{L}_{\text{SFT}} = \sum_t \nabla_\theta \log \pi_\theta(a_t | s_t) \cdot Q^{\text{SFT}}(s_t, a_t)
\end{equation}
where $Q^{\text{SFT}}(s_t, a_t)$ represents the "future value" from taking action $a_t$ at state $s_t$. Since each token contributes gradient $\nabla_\theta \log \pi = \nabla_\theta \pi / \pi$, the implicit reward is effectively $1/\pi$ weighted. Working backwards from the terminal state:

At the final decision (STOP at step $T$):
\begin{equation}
    Q^{\text{SFT}}_d(s_T, \text{STOP}) = 0
\end{equation}
At the final sampling step $T$, token $j$:

\begin{equation}
\resizebox{1.0\hsize}{!}{$
Q^{\text{SFT}}_{\text{sample}}(s_{T-1}, \text{token}_{T,j}) = \sum_{j'=j+1}^{L_T} \frac{1}{\pi_{\text{sample}}(\text{token}_{T,j'})} + \frac{1}{\pi_d(\text{STOP} \mid s_T)}
$}
\end{equation}

Recursively, for step $k < T$:
\begin{equation}
\resizebox{1.0\hsize}{!}{$
\begin{aligned}
Q^{\text{SFT}}_d(s_k, \text{RESAMPLE}) &= Q^{\text{SFT}}_{\text{sample}}(s_k, \text{token}_{k+1,1}) \\
Q^{\text{SFT}}_{\text{sample}}(s_{k-1}, \text{token}_{k,j}) &= \sum_{j'=j+1}^{L_k} \frac{1}{\pi_{\text{sample}}(\text{token}_{k,j'})} + \frac{1}{\pi_d(a_k \mid s_k)} + Q^{\text{SFT}}_d(s_k, a_k)
\end{aligned}
$}
\end{equation}
The policy-entanglement is evident: $Q^{\text{SFT}}_{\text{sample}}$ contains $1/\pi_d$ terms, and $Q^{\text{SFT}}_d$ contains $1/\pi_{\text{sample}}$ terms through the recursive dependency. These Q-values change as $\pi_\theta$ updates during training, preventing clean credit assignment.
\end{proof}

\begin{corollary}
No unified sufficient statistic $\Phi$ exists for SFT. Following the logic of Theorem 3.2, define:
\begin{align}
    \Phi^{\text{SFT}}_{\text{sample}}(s_k) &= f\left(\frac{1}{\pi_{\text{sample}}}, \frac{1}{\pi_d}, L\right) \\
    \Phi^{\text{SFT}}_d(s_k) &= g\left(\frac{1}{\pi_{\text{sample}}}, \frac{1}{\pi_d}, L\right)
\end{align}
Both depend on the full policy $\pi_\theta$, making $\Phi^{\text{SFT}}_{\text{sample}} \neq \Phi^{\text{SFT}}_d$ in general.
\end{corollary}

\subsection{DFT: Removing Policy-Entanglement}
Dynamic Fine-Tuning rescales each token's contribution by its probability:
\begin{equation}
\resizebox{1.0\hsize}{!}{$
\begin{aligned}
\mathcal{L}_{\text{DFT}}(\theta) &= \mathbb{E}_{\tau \sim \mathcal{D}}\Bigg[\sum_{k=1}^{T} \sum_{j=1}^{L_k} \pi_{\text{sample}}(\text{token}_{k,j}) \cdot \log \pi_{\text{sample}}(\text{token}_{k,j}) \\
&\quad + \sum_{k=1}^{T} \pi_d(a_k \mid s_k) \cdot \log \pi_d(a_k \mid s_k)\Bigg]
\end{aligned}
$}
\end{equation}

\begin{lemma}[Cancellation of $1/\pi$ weighting]
The DFT gradient satisfies:
\begin{equation}
    \pi \cdot \nabla_\theta \log \pi = \pi \cdot \frac{\nabla_\theta \pi}{\pi} = \nabla_\theta \pi
\end{equation}
Therefore:
\begin{equation}
\resizebox{1.0\hsize}{!}{$
\begin{aligned}
\nabla_\theta \mathcal{L}_{\text{DFT}} &= \mathbb{E}_{\tau \sim \mathcal{D}}\Bigg[\sum_{k=1}^{T} \sum_{j=1}^{L_k} \nabla_\theta \pi_{\text{sample}}(\text{token}_{k,j}) \\
&\quad + \sum_{k=1}^{T} \nabla_\theta \pi_d(a_k \mid s_k)\Bigg]
\end{aligned}
$}
\end{equation}
The $1/\pi$ weighting is exactly canceled.
\end{lemma}

\begin{theorem}[DFT Removes Policy-Entanglement]
For DFT, the implicit Q-values become policy-independent:
\begin{equation}
\resizebox{1.0\hsize}{!}{$
\begin{aligned}
Q^{\text{DFT}}_{\text{sample}}(s_{k-1}, \text{token}_{k,j}) &= (L_k - j) \cdot c + c + \sum_{k'=k+1}^{T}(L_{k'} + 1) \cdot c \\
Q^{\text{DFT}}_d(s_k, a_k) &= \sum_{k'=k+1}^{T}(L_{k'} + 1) \cdot c
\end{aligned}
$}
\end{equation}
where $c$ is a constant independent of $\pi_\theta$.
\end{theorem}

\begin{proof}
With the $1/\pi$ weighting cancelled, each token contributes constant implicit reward $c$. Working backwards:

At the final decision:
\begin{equation}
    Q^{\text{DFT}}_d(s_T, \text{STOP}) = 0
\end{equation}
At step $T$, token $j$:
\begin{equation}
    Q^{\text{DFT}}_{\text{sample}}(s_{T-1}, \text{token}_{T,j}) = (L_T - j) \cdot c + c
\end{equation}
where $(L_T - j) \cdot c$ accounts for remaining sampling tokens and $c$ accounts for the final decision. Recursively, for step $k < T$:
\begin{equation}
\resizebox{1.0\hsize}{!}{$
Q^{\text{DFT}}_d(s_k, \text{RESAMPLE}) = L_{k+1} \cdot c + c + Q^{\text{DFT}}_d(s_{k+1}, a_{k+1})
$}
\end{equation}
Expanding the recursion:
\begin{equation}
\resizebox{1.0\hsize}{!}{$
Q^{\text{DFT}}_d(s_k, a_k) = \sum_{k'=k+1}^{T} L_{k'} \cdot c + (T - k) \cdot c = \sum_{k'=k+1}^{T}(L_{k'} + 1) \cdot c
$}
\end{equation}
Crucially:
\begin{itemize}
    \item $Q^{\text{DFT}}_{\text{sample}}$ does not depend on $\pi_d$.
    \item $Q^{\text{DFT}}_d$ does not depend on $\pi_{\text{sample}}$.
\end{itemize}
The Q-values are now policy-independent.
\end{proof}

\subsection{DFT: Remaining Length Asymmetry}
Although DFT removes policy-entanglement, length asymmetry persists:
\begin{equation}
    Q^{\text{DFT}}_d(s_k, a_k) = c \cdot \sum_{k'=k+1}^{T}(L_{k'} + 1) \approx c \cdot \sum_{k'=k+1}^{T} L_{k'}
\end{equation}
Since $L_{k'} \gg 1$ (sampling sequences contain hundreds of tokens while decisions are single tokens), we have:
\begin{equation}
    Q^{\text{DFT}}_d \sim O\left(\sum_{k'} L_{k'}\right) \gg Q^{\text{DFT}}_{\text{sample},\text{per-token}} \sim O(1)
\end{equation}
The decision policy gradient is weighted by accumulated future sequence lengths, even though each decision is a single token. This creates asymmetric regularization similar to (but weaker than) the KL penalty analyzed in Theorem 3.2.

\begin{table*}[t!]
\centering
\begin{tabular}{lccc}
\toprule
\textbf{Property} & \textbf{SFT} & \textbf{DFT} & \textbf{RL (Surrogate Reward)} \\
\midrule
$Q_d$ depends on $\pi_{\text{sample}}$ & Yes (via $1/\pi_{\text{sample}}$) & No & No \\
$Q_{\text{sample}}$ depends on $\pi_d$ & Yes (via $1/\pi_d$) & No & No \\
$Q_d$ depends on future lengths $L_{k'}$ & Yes (weighted by $1/\pi$) & Yes (constant weight) & No \\
Sufficient statistic & $\Phi^{\text{SFT}}_{\text{sample}} \neq \Phi^{\text{SFT}}_d$ & $\Phi^{\text{DFT}}_{\text{sample}} \neq \Phi^{\text{DFT}}_d$ & $\Phi = \gamma^{\sum len} A_i$ (unified) \\
Gradient attribution & Imbalanced & Improved & Balanced \\
\bottomrule
\end{tabular}
\caption{Comparison of gradient attribution properties.}
\end{table*}

\subsection{Graphical Illustration}
We parameterize a minimal two-stage process with three learnable logits: $\theta_s$ governing sampling accuracy, $\theta_{d|C}$ governing stop decisions given correct answers, and $\theta_{d|W}$ governing resample decisions given incorrect answers.

\begin{figure*}[h!]
    \centering
    \includegraphics[width=0.9\linewidth]{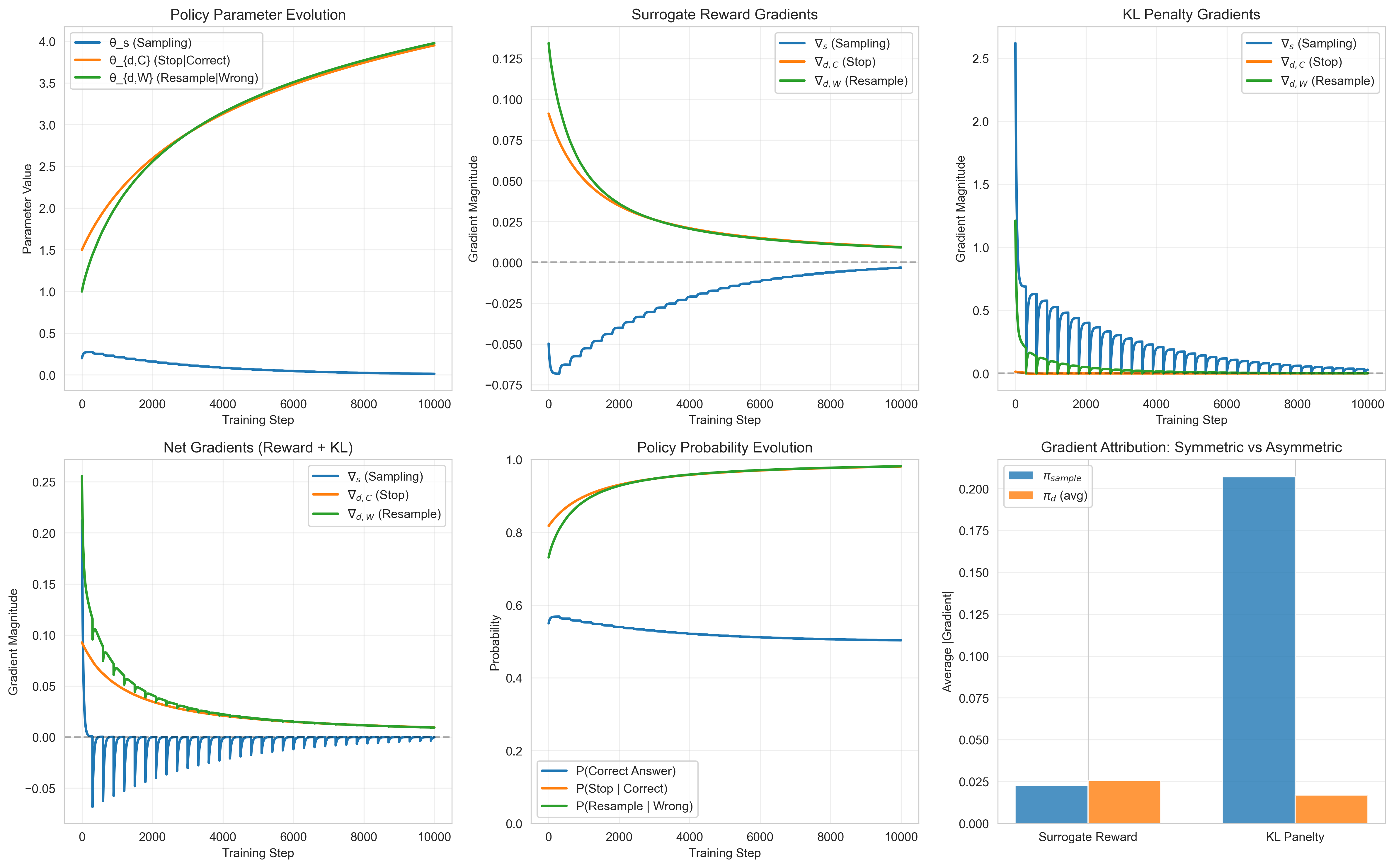}
    \caption{Single Parameter Illustration}
    \label{fig:simple_simulation}
\end{figure*}

Figure~\ref{fig:simple_simulation}\footnote{The cyclical pattern observed in the training process is because of reference model updating.} displays the training dynamics under the combined GRPO objective. The center panels decompose gradient contributions: the surrogate reward generates comparable gradient magnitudes across all parameters (center-left), while the KL penalty produces gradients for sampling that dominate decision gradients by nearly an order of magnitude (center-right). This confirms the $O(L_k)$ versus $O(1)$ asymmetry established above in the section. The bottom-right panel quantifies this directly—gradient attribution remains symmetric for the surrogate reward but exceeds 2:1 (sampling versus decision) for the KL penalty throughout training.

The net effect (bottom-left) is differential learning: $\pi_d$ parameters receive sustained positive gradients while $\pi_{\text{sample}}$ is heavily regularized toward the reference.


\section{Task Description and Experiment Hyperparameters}
\label{app:experiment_details}

\subsection{Task Description}

We evaluate models on multi-digit multiplication tasks, denoted as $m \times n$ where $m$ and $n$ indicate the number of digits in each operand. For example, a $3 \times 4$ task involves multiplying a 3-digit number by a 4-digit number (e.g., $123 \times 4567$). 

Each problem is presented to the model as a natural language query:
\begin{quote}
Calculate [operand1] * [operand2]. Think step by step.
\end{quote}
An answer is scored as correct if and only if the final numerical output exactly matches the ground truth product. Intermediate reasoning steps are not evaluated for correctness.

\subsection{Dataset Construction}

\paragraph{Training Data.} We construct 20,000 training examples for the in-distribution tasks ($4 \times 5$ and $5 \times 4$ multiplication). Operands are sampled uniformly at random within the specified digit range: for an $m$-digit operand, we sample integers from $[10^{m-1}, 10^m - 1]$. The training set is balanced equally between $4 \times 5$ and $5 \times 4$ tasks.

\paragraph{Test Data.} For each difficulty level ($3 \times 3$ through $3 \times 9$), we construct a test set of 100 examples using the same uniform sampling procedure. No filtering or balancing is applied for the results reported in Table~\ref{tab:accu_over_tasks}. For the calibration analysis in Figure~\ref{fig:calibration_comparison}, we exclude samples that exceed the 8,192 token generation limit.

\paragraph{SFT Data Variants.} We construct two SFT training sets:
\begin{itemize}
    \item \textbf{SFT (no reflection):} Each example consists of a query paired with a chain-of-thought solution leading directly to the answer, without retry patterns.
    \item \textbf{SFT (reflection):} Each example consists of a query paired with a trajectory that may include multiple attempts, explicit error detection, and correction steps before arriving at the final answer.
\end{itemize}

Trajectory generation procedures are provided in the replication repository.

\subsection{Model and Training Details}

All experiments use Qwen2.5-7B-Instruct \citep{qwen2025qwen25technicalreport} as the base model.

\paragraph{SFT Training.} Supervised fine-tuning is conducted using LLaMA-Factory \citep{zheng2024llamafactory}. Hyperparameters are as follows:
\begin{table}[H]
\centering
\begin{tabular}{ll}
\toprule
\textbf{Hyperparameter} & \textbf{Value} \\
\midrule
Optimizer & Adam \\
Batch size & 128 \\
Learning rate & $1 \times 10^{-5}$ \\
Epochs & 4.0 \\
Training examples & 20,000 \\
\bottomrule
\end{tabular}
\caption{SFT training hyperparameters.}
\label{tab:sft_hyperparams}
\end{table}

\paragraph{RL Training.} Reinforcement learning is conducted using verl \citep{sheng2024hybridflow} with Group Relative Policy Optimization (GRPO). Hyperparameters are as follows:
\begin{table}[H]
\centering
\begin{tabular}{ll}
\toprule
\textbf{Hyperparameter} & \textbf{Value} \\
\midrule
Algorithm & GRPO \\
Learning rate & $1 \times 10^{-6}$ \\
Train batch size & 512 \\
PPO mini-batch size & 128 \\
Rollouts per query ($n$) & 8 \\
KL loss coefficient & 0.001 \\
KL loss type & low\_var\_kl \\
Max prompt length & 1,024 tokens \\
Max response length & 8,192 tokens \\
Total epochs & 2 \\
Gradient checkpointing & Enabled \\
\bottomrule
\end{tabular}
\caption{RL training hyperparameters.}
\label{tab:rl_hyperparams}
\end{table}

\subsection{Evaluation Protocol}

All models are evaluated with the following settings:
\begin{itemize}
    \item \textbf{Sampling temperature:} 1.0
    \item \textbf{Number of attempts:} 1 (single generation per query)
    \item \textbf{Accuracy metric:} $\frac{\text{Number of correct answers}}{\text{Total number of questions}}$
\end{itemize}

\paragraph{Confidence Intervals.} Confidence intervals for observed accuracy in Table~\ref{tab:accu_over_tasks} are computed using the normal approximation to the binomial distribution at the 95\% level. Confidence intervals for predicted accuracy in Figure~\ref{fig:calibration_rl} and Figure~\ref{fig:calibration_sft} are obtained via bootstrapping: we resample the test set with replacement 100 times, compute the predicted accuracy for each bootstrap sample using the calibrated $p_s$, $p_{d|C}$, and $p_{d|W}$ parameters, and report the 2.5th and 97.5th percentiles of the resulting distribution.

\subsection{Compute Resources}

Each experiment is conducted on a cloud server equipped with a single NVIDIA H100 GPU.

\end{document}